\documentclass[letterpaper, 10 pt, conference]{ieeeconf}  
\IEEEoverridecommandlockouts                              
\usepackage{amsmath}
\usepackage{amssymb}
\usepackage{algorithm}
\usepackage[noend]{algpseudocode}
\usepackage{listings}
\usepackage{tikz}
\usetikzlibrary{trees}
\usetikzlibrary{calc}
\usetikzlibrary{decorations.shapes}
\usetikzlibrary{arrows}

\usepackage{graphicx} 
\usepackage{tabularx,calc}
\usepackage[numbers]{natbib}
\usepackage[bookmarks=true]{hyperref}

\usepackage{amsfonts} 

\usepackage{amsthm,amssymb}
\usepackage{etex,etoolbox} 
\usepackage{environ}
\makeatletter

\providecommand{\@fourthoffour}[4]{#4}
\def\fixstatement#1{%
  \AtEndEnvironment{#1}{%
    \xdef\pat@label{\expandafter\expandafter\expandafter
      \@fourthoffour\csname#1\endcsname\space\@currentlabel}}}

\globtoksblk\prooftoks{1000}
\newcounter{proofcount}

\NewEnviron{proofatend}{%
  \it{Proof in Appendix.}
  \edef\next{%
    \noexpand\begin{proof}[Proof of \pat@label]%
    \unexpanded\expandafter{\BODY}}%
  \global\toks\numexpr\prooftoks+\value{proofcount}\relax=\expandafter{\next\end{proof}}
  \stepcounter{proofcount}}

\def\printproofs{%
  \count@=\z@
  \loop
    \the\toks\numexpr\prooftoks+\count@\relax
    \ifnum\count@<\value{proofcount}%
    \advance\count@\@ne
  \repeat}

\makeatother

\theoremstyle{definition}
\theoremstyle{plain}

\newtheorem{definition}{Definition}
\newtheorem{theorem}{Theorem}

\fixstatement{theorem}
\fixstatement{lemma}
\fixstatement{corollary}

\algnewcommand{\And}{\textbf{and}}

\usepackage{xspace}

\newcommand{\bigslant}[2]{{\raisebox{.2em}{$#1$}\left/\raisebox{-.2em}{$#2$}\right.}}
\def\R{\ensuremath{\mathbb{R}}}
\def\N{\ensuremath{\mathbb{N}}}
\def\env{\ensuremath{\mathbf{E}}}

\def\Q{\ensuremath{\mathbf{Q}}}

\def\G{\ensuremath{\mathbf{G}}}
\def\C{\ensuremath{\mathbf{C}}}

\def\myplanner{\ensuremath{\texttt{QMP}}\xspace}

\newcommand\SE{ \ensuremath{\operatorname{SE}}\xspace}
\newcommand\SO{ \ensuremath{\operatorname{SO}}\xspace}

\def\Robot{\ensuremath{\mathcal{R}}\xspace}

\def\qr{q_{rand}}
\def\qn{q_{near}}
\def\qw{q_{new}}

\def\N{\ensuremath{\mathbb{N}}}
\def\M{\ensuremath{\mathcal{M}}}

\def\V{\ensuremath{\mathbf{V}}}

\def\klampt{\texttt{Klamp't}\xspace}
\def\ompl{\texttt{OMPL}\xspace}

\def\prm{\texttt{PRM}\xspace}

\def\rrtconnect{\texttt{RRTConnect}\xspace}

\def\est{\texttt{EST}\xspace}

\def\path{\ensuremath{\mathbf{p}}}

\title{\Huge \bf
Quotient-Space Motion Planning
}
\author{Andreas Orthey\thanks{The  authors  are  with  the  CNRS-AIST  Joint  Robotics  Laboratory,
UMI3218/RL, Tsukuba-shi 305-8560, Japan (e-mail: \{andreas.orthey, adrien.escande, e.yoshida\}@aist.go.jp} and Adrien Escande and Eiichi Yoshida
\thanks{This research was supported by Grant-in-Aid for JSPS Research Fellow 16F116701}
}

\begin{document}

\maketitle
\thispagestyle{empty}
\pagestyle{empty}

\begin{abstract}

  A motion planning algorithm computes the motion of a robot by computing a path
through its configuration space. To improve the runtime of motion planning
algorithms, we propose to nest robots in each other, creating a nested quotient-space decomposition of the configuration space.
  Based on this decomposition we define a new roadmap-based motion planning algorithm called
  the Quotient-space roadMap Planner (QMP). The algorithm starts growing a graph on the lowest dimensional quotient space, switches to the next quotient space once a valid path has been
  found, and keeps updating the graphs on each quotient space simultaneously until a valid path in the configuration space has been found.  We show that this algorithm is probabilistically complete and outperforms a set of state-of-the-art algorithms
  implemented in the open motion planning library (OMPL).

\end{abstract}

\section{Introduction}

Motion planning algorithms are fundamental for robotic applications like product
assembly, manufacturing, disaster response, elderly care or space exploration.

A motion planning algorithm takes as input a robot, its configuration space, an environment, a start and a goal configuration, and computes as output a path between start and goal if one exists \cite{lavalle_2006}. This computation is NP-hard \cite{reif_1979} scaling exponentially with the number of dimensions of the configuration space. Thus, the more degrees-of-freedom (dof) a robot has, the higher the runtime of the motion planning algorithm will be. This can become the limiting factor of any robotics application. 
It is therefore important to find suitable decompositions of the configuration space, such that a planning algorithm can quickly discover the relevant parts of the configuration space. 

We developed a new decomposition of a given configuration space $\M$, which decomposes $\M$ into a sequence of nested subspaces. We observe that any configuration space $\M$ can be written as a product of subspaces
\begin{equation}
  \begin{aligned}
   \M = \M_1 \times \cdots \times \M_K
  \end{aligned}
\end{equation}

\noindent This suggests we can decompose the configuration space in the following way: Start with the product of subspaces and successively remove one subspace after another. This leads to a sequence of nested subspaces as
\begin{equation}
  \begin{aligned}
   \M_1 \subset \M_1 \times \M_2 \subset \cdots \subset \M_1 \times \cdots \times \M_K
  \end{aligned}
\end{equation}

\noindent Each subspace in this sequence is called a quotient-space, and the sequence itself is called a quotient-space decomposition \cite{munkres_2000}.

It turns out that each quotient-space can be represented by nesting a simpler robot inside the original robot. The prototypical example is a rigid body free-floating in space. The configuration space is $\SE(3) = \R^3 \times \SO(3)$, and we can decompose it as $\R^3 \subset \R^3 \times \SO(3)$. The subspace $\R^3$ is called the quotient-space and represents a sphere nested inside the rigid body, abstracting the orientations of the rigid body.

Such a decomposition is advantageous: Imagine the sphere being infeasible at point $p \in \R^3$. Then the rigid body is infeasible at all configurations inside the subspace $p \times \SO(3)$. We call this the necessary condition of nested robots. 



We have developed a new algorithm, called the Quotient-space roadMap Planner (\myplanner), which is able to exploit such a quotient-space decomposition.
We first show how to build a quotient-space decomposition for any robot in Sec. \ref{sec:background}. We then discuss the inner workings of \myplanner in Sec. \ref{sec:qmp}, prove that \myplanner is probabilistically complete (Sec. \ref{sec:completeness}), and we develop three heuristics to improve its runtime (Sec. \ref{sec:heuristics}). Finally, we demonstrate that \myplanner (Sec. \ref{sec:experiments}) can be applied to free-floating rigid bodies, free-floating articulated bodies, fixed-base serial chains and fixed-base tree chain robots. 
\section{Related Work}



%




We review two categories of papers. First, we review quotient-space decompositions and their application to motion planning. Second, we review alternative decomposition methods.

Quotient-space decompositions are ubiquitous in mathematics, appearing as quotient-groups in algebra, filtrations in linear algebra, or nests in functional analysis. The construction of simplicial complexes in algebraic topology is a prominent example of a quotient-space decomposition.

The application of quotient spaces to continuous spaces and decision making has
been originally developed by \cite{zhang_2004} and \cite{yao_2013}. In robotics, quotient-spaces have been used, albeit under different names. \citeauthor{bretl_2006} proposed a two-level decomposition: first a path on a stance graph is planned and then
configurations along the path are sampled \cite{bretl_2006}. The stance graph can be seen as a quotient space of the configuration space by the stance subspaces. Grey et al. \cite{grey_2017} use a two-level quotient space decompositions, embedding the torso inside a humanoid robot. A similar idea can be found in Tonneau et al.\cite{tonneau_2018}, approximating a robot by a simpler model. Both methods are similar to ours, but use only a two-level decomposition, and do not define the nesting procedure for general robots.

The closest approach to ours is the multi-level decomposition scheme by \cite{zhang_2009} whereby a sequence of nested robots is created. Their planning algorithm starts with the lowest dof robot, computes a path, and uses this path as a constraint for the next bigger robot. This approach can neither deal with non-simple paths (see Fig.
\ref{fig:nonsimple}), nor with spurious paths (see Fig. \ref{fig:spurious}, Sec. \ref{sec:heuristics}). It is thus not complete.

Closely related is the exploration/exploitation tree (EET) algorithm \cite{rickert_2014}. The authors compute a sequence of spheres in the workspace, called a tunnel, to approximate the space of collision-free paths. However, this method is not complete, and cannot handle spurious shortest tunnels/paths (see Sec. \ref{sec:heuristics}).

The motion planning problem can also be decomposed without using quotient-spaces.

Kunz et al. \cite{kunz_2016} use a hierarchical rejection sampling
approach to improve the Informed-RRT* 
algorithm \cite{gammell_2014}. While the focus is different from ours, their method is methodically similar to our algorithm, where samples are discarded if they are not necessary feasible. 

Our work is closely related to low-dimensional sampling techniques which guide configuration space sampling. The algorithm by van den Berg and Overmars \cite{berg_2005} precomputes narrow passages of the workspace and uses thoses passages to sample the space more densely in those areas. A similar idea can be found in \cite{zucker_2008}, where the authors discretize the workspace, compute a shortest path in the workspace between start and goal configuration, and then sample from a cell in workspace in proportion to the cell's distance to the shortest path. Closely related is also the dimensionality reduction method by \cite{orthey_2015}, which considers ignoring paths not having a minimal swept volume. 

The KPIECE algorithm \cite{sucan_2009} is another example of a hierarchical decomposition. The environment is divided into smaller and smaller boxes until a certain threshold is reached. A small box corresponds to workspace points near to a boundary. Those areas are sampled more frequently to effectively guide samples towards the configuration space boundary. This algorithm is orthogonal to ours: they decompose the environment, we decompose the robot.

\section{Quotient-Space Decomposition\label{sec:background}}

First, we describe the idea of a quotient space and show two applications to the vector space $\R^2$ and the manifold $\SO(2)^2$. Second, we show how robots can be nested in each other and thereby create a sequence of quotient spaces. Third, we show that being feasible in a quotient space is a necessary condition for being feasible in the configuration space. 

\subsection{Quotient Space}

Let $M$ be a vector space and $C$ be a subspace of $M$. Then the quotient space of $M$ by $C$, denoted by $\bigslant{M}{C}$ is the space obtained by collapsing all equivalence classes of $C$ in $M$ to zero \cite{munkres_2000}. Collapsing a space is done by creating an equivalence relation $\sim$ on $M$: for all $x,y \in M$ we have that $x\sim y$ if $x-y \in C$. This relation creates a partition of equivalence classes on the vector space. The set of equivalence classes is called the quotient space.

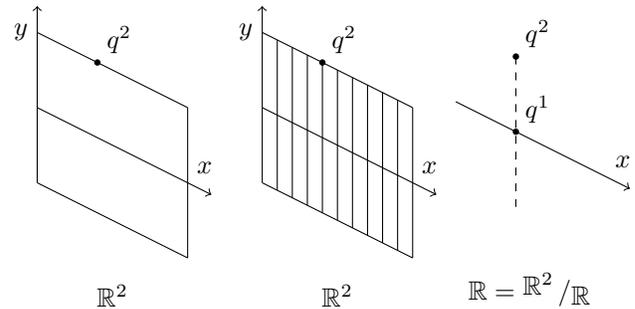
\begin{figure}
  \centering
  \newcommand\plane[1]{
  \coordinate (A) at (0,0);
\coordinate (B) at (0,-2);
\coordinate (C) at (2,-1);
\coordinate (D) at (2,-3);
\coordinate (E) at (0,-1);
\coordinate (F) at (2,-2);
\draw[<-,shorten <= -10pt] (A) -- (B);
\draw (B) -- (D);
\draw (D) -- (C);
\draw (C) -- (A);
\draw[->,shorten >= -10pt] (E) -- (F);
\node[left] at (A){$y$};
\node[above right] at (F){$x$};
\coordinate (Q) at (0.8,-0.4);
\node[above right] at (Q){$q^2$};
\draw[black,fill=black,circle] (Q) circle (1pt);
\coordinate (G) at (1,-3.5);
\node at (G){#1};
}
\begin{tikzpicture}
\plane{$\R^2$}
\end{tikzpicture}
\begin{tikzpicture}
\plane{$\R^2$}
\foreach \x in {1,...,9}{
   \coordinate (\x0) at (\x*0.2,-\x*0.1);
   \coordinate (\x1) at (\x*0.2,-\x*0.1-2);
   \draw (\x0) -- (\x1);
   }
\end{tikzpicture}
\begin{tikzpicture}
\coordinate (E) at (0,-1);
\coordinate (F) at (2,-2);
\coordinate (G) at (1,-3.5);
\coordinate (Q) at (0.8,-1.4);
\coordinate (x0) at (0.8,-0.4);
\coordinate (x1) at (0.8,-2.4);
\draw[->,shorten >= -10pt] (E) -- (F);
\node[above right] at (F){$x$};
\node at (G){$\R=\bigslant{\R^2}{\R}$};
\draw[black,fill=black,circle] (Q) circle (1pt);
\draw[black,fill=black,circle] (x0) circle (1pt);
\node[above right] at (Q){$q^1$};
\node[above right] at (x0){$q^2$};
\draw[dashed] (x0) -- (x1);
\end{tikzpicture}
\vspace{-10pt}
  \caption{Visualization of Quotient Space as collapsing of equivalence classes.\label{fig:cosets}}
  \vspace{-5pt}
\end{figure}

\begin{figure}
  \centering
  \def\lwd{0.49}
  \def\lhd{0.49}
  \includegraphics[width=\lwd\linewidth,height=\lhd\linewidth]{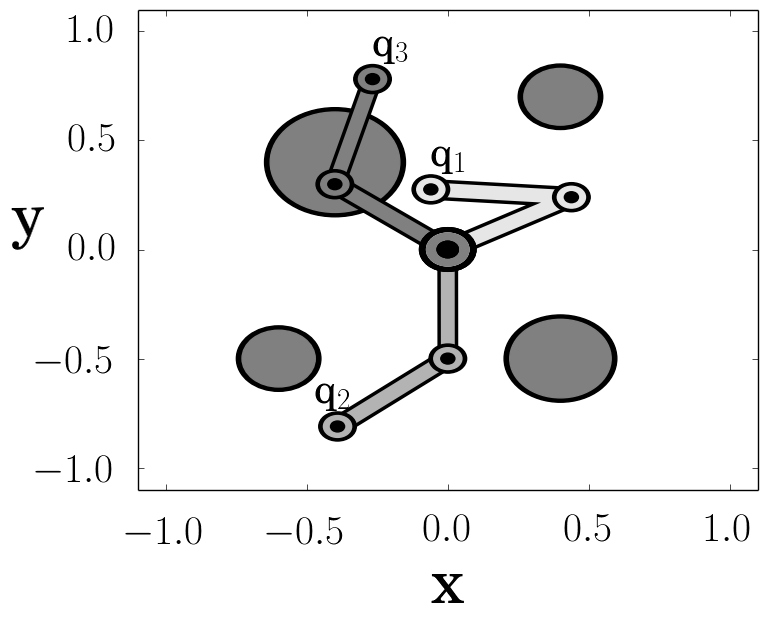}
  \includegraphics[width=\lwd\linewidth,height=\lhd\linewidth]{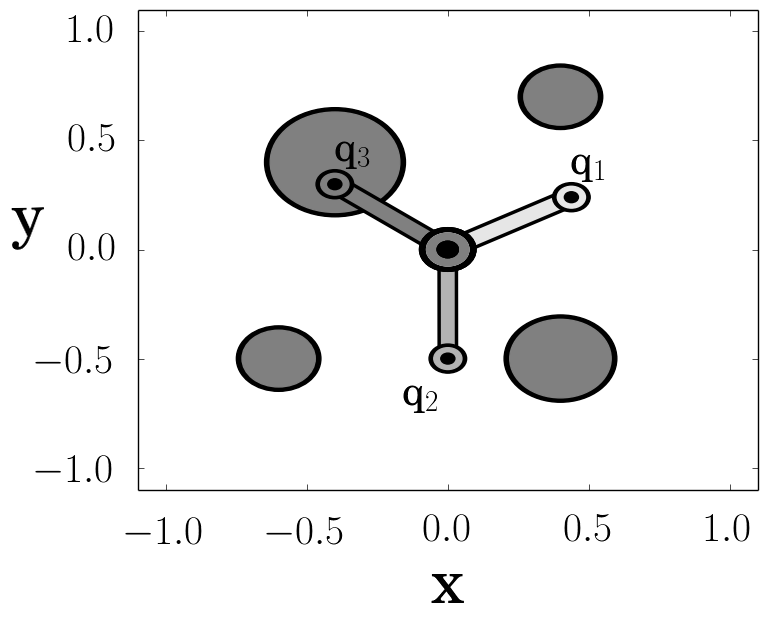}
  \includegraphics[width=\lwd\linewidth,height=\lhd\linewidth]{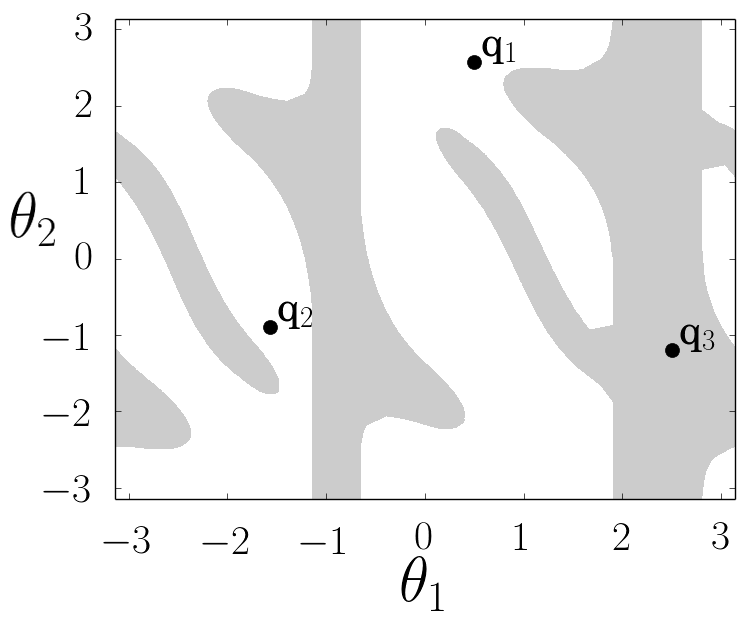}
  \includegraphics[width=\lwd\linewidth,height=\lhd\linewidth]{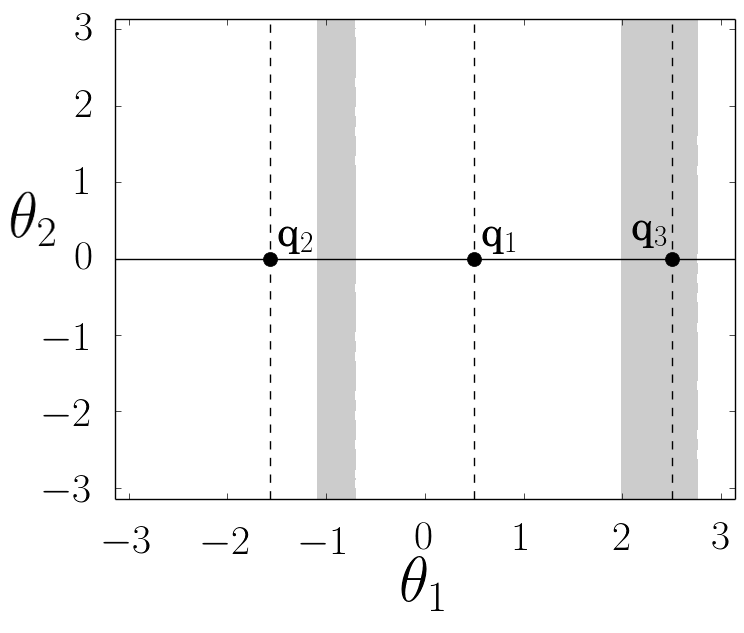}
  \vspace{-5pt}
  \begin{tikzpicture}
\coordinate (G) at (0,0);
\node at (G){$\M_1=\SO(2) \times \SO(2)$};
\coordinate (G) at (4.5,0);
\node at (G){$\M_0=\bigslant{\M_1}{\SO(2)}$};
\end{tikzpicture}
\vspace{-10pt}
  \caption{Quotient Space Decomposition for a 2-dof manipulator. Left Top: Environment with four obstacles. We have visualized three distinct configurations of the robot. Left Bottom: The configuration space with gray areas being infeasible configurations. Right Top: Environment with 1-dof nested manipulator at same configurations. Right Bottom: Configuration space is collapsed onto the quotient space. Being feasible in the quotient space is a necessary condition for feasibility in the configuration space. It can be seen that $q_3$ is infeasible, thus not fulfilling the necessary condition. \label{fig:2dof}}
  \vspace{-10pt}
\end{figure}

As an example, consider the vector space $\R^2$ and its subspace $\R$. We can partition $\R^2$ into equivalence classes of its subspace $\R$ such that two points $x,y \in \R^2$ are equivalent if $x-y \in 0 \times \R$. We visualize this in Fig. \ref{fig:cosets}, where $\R^2$ is first partitioned into the equivalence classes of $\R$ (the vertical lines), and then all lines are collapsed to yield $\R$. The point $q^2 \in \R^2$ and all points on the dashed line are equivalent and therefore collapsed into the single point $q^1 \in \R$. We identify $\R = \bigslant{\R^2}{\R}$ to denote the quotient space.

Quotient spaces generalize to manifolds \cite{lee_2003}. Consider $\M = \SO(2) \times \SO(2)$, the configuration space of a 2-dof fixed-base manipulator in the plane. Two points $x,y \in \SO(2)$ are equivalent if $x-y \in 0 \times \SO(2)$. The manifold $\M$ can be partitioned into equivalence classes of $\SO(2)$, and then each equivalent class is collapsed to yield $\SO(2) = \bigslant{\M}{\SO(2)}$. This has been visualized in Fig. \ref{fig:2dof}. On the left three different configurations are shown and their position in the configuration space. On the right the quotient space is shown, corresponding to a 1-dof manipulator nested inside the 2-dof manipulator. Two interesting observations can be made in the quotient space: First, the configuration $q_3$ is infeasible, regardless of how the second joint is moved. Second, there does not exists a path between $q_1$ and $q_2$. This can be inferred solely from the 1-dof manipulators configuration space.


This example shows the defining feature of quotient spaces: If a robot is nested inside another robot, then the feasibility of the nested robot is a necessary condition for the feasibility of the other robot\footnote{Zhang and Zhang \cite{zhang_2004} call this the falseness-preserving property of quotient space decompositions}. We will proceed to define how robots can be nested in each other, and we will show that this definition indeed leads to the necessary condition for feasibility.

%
%

\subsection{Nesting of Robots\label{sec:nestedrobots}}

A robot $\Robot_i$ is nested inside another robot $\Robot_{i+1}$ if two conditions are fulfilled: First, the configuration space of $\Robot_i$ is a subspace of $\Robot_{i+1}$, and second, the volume of the body of $\Robot_{i}$ at each configuration must be a subset of the volume of the body of $\Robot_{i+1}$.

Let $\M_i$ be the configuration space of robot $\Robot_i$, and let $\V_i(p) \subset \R^3$ be
the volume of the body of robot $\Robot_i$ at configuration $p \in \M_i$. 

\begin{definition}[Nested Robot]

  Let $\Robot_i, \Robot_{i+1}$ be given. We say that $\Robot_{i}$ is nested
  in $\Robot_{i+1}$, denoted as $\Robot_i \subseteq \Robot_{i+1}$, if 
\begin{enumerate}
  \item[{(1)}] $\M_{i+1}=\M_i \times \C_{i+1}$ such that $\M_{i}=\bigslant{\M_{i+1}}{\C_{i+1}}$
  \item[{(2)}] $\V_i(p) \subseteq \V_{i+1}(p\circ q)$ for any $p
  \in \M_i$ and $q \in \C_{i+1}$\label{def:nesting}
  \end{enumerate}
\end{definition}
whereby the operation $\circ$ is the cartesian product defined as $p \circ q = (p,q) \in \M_{i+1}$.

Given a robot $\Robot$ and a sequence of nested robots $\Robot_1 \subseteq \cdots \subseteq \Robot_K = \Robot$, the configuration spaces define a decomposition of $\M$ as
\begin{equation}
  \begin{aligned}
    \M_1 \subseteq \cdots \subseteq \M_K
    \label{eq:filtration}
  \end{aligned}
\end{equation}

\noindent This decomposition will be called a quotient-space decomposition. 

\subsection{Necessary Conditions}

From Definition \ref{def:nesting} we can infer the key property of the quotient-space decomposition: if a nested robot is infeasible, then so is the original robot. 

Let $\env$ denote the environment, the subset of $\R^3$ containing obstacles. We say that robot $\Robot_i$ is feasible at configuration $q \in \M_i$ if $\V_i(q) \cap \env = \emptyset$.

\begin{theorem}[Necessary condition of nested robot]
If $\Robot_i \subseteq \Robot_{i+1}$, $p \in \M_i$, and $\V_i(p) \cap \env \neq \emptyset$ (robot $\Robot_i$ is infeasible at $p$), then $\V_{i+1}(p\circ q) \cap \env \neq \emptyset$ for any $q \in \C_{i+1}$ (robot $\Robot_{i+1}$ is infeasible  everywhere).\label{thm:necessary}
\end{theorem}
\begin{proof}
Since $\V_i(p) \cap \env \neq \emptyset$ and $\V_i(p) \subseteq \V_{i+1}(p \circ q)$ for any $q \in \C_{i+1}$, it must follow that $\V_{i+1}(p \circ q) \cap \env \neq \emptyset$.
\end{proof}

\noindent Theorem \ref{thm:necessary} implies that if a configuration $p \in \M_i$ is infeasible, then so is the subspace $p \times \C_{i+1}$. We can exploit this fact to ignore subspaces of the configuration space, thereby developing a motion planning algorithm with lower runtime.


\section{Quotient-Space RoadMap Planner\label{sec:qmp}}

The Quotient-Space roadMap Planner (\myplanner) works as depicted in Fig. \ref{fig:qmp}. On the top right the configuration space of the 2-dof fixed-base manipulator is shown with a start configuration (green) and a goal configuration (red). The corresponding start and goal configurations on the quotient-space are shown in the top left figure. In the first step, a graph is grown on the quotient space (middle left). Once a valid path has been found between start and goal configuration, a second graph is grown on the configuration space (middle right), whereby the samples are constrained to lie above the quotient-space graph. Both graphs are simultaneously grown (bottom left),
until a valid path has been found on the configuration space (bottom right), or until a time limit has been reached. For more than two quotient-spaces, this idea is iteratively continued. 

\begin{figure}
\def\lw{0.49}
\def\lh{0.48}
\includegraphics[width=\lw\linewidth,height=\lh\linewidth]{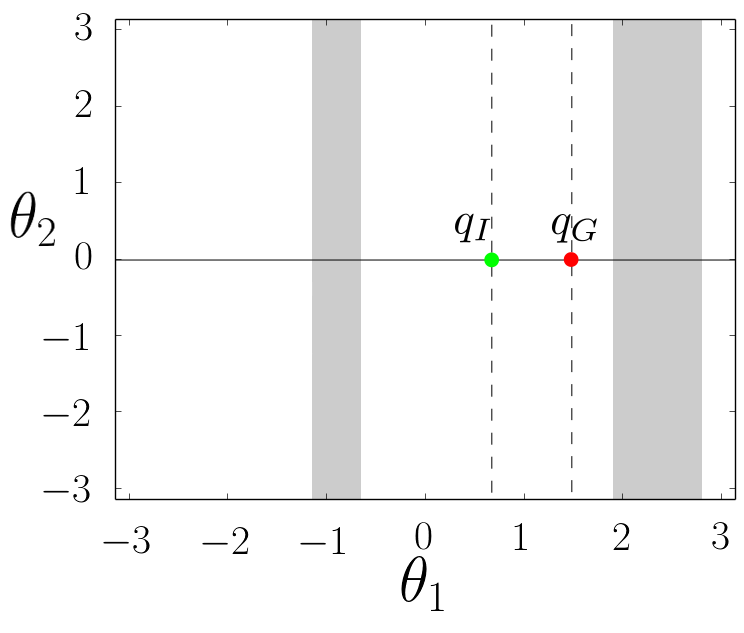}
\includegraphics[width=\lw\linewidth,height=\lh\linewidth]{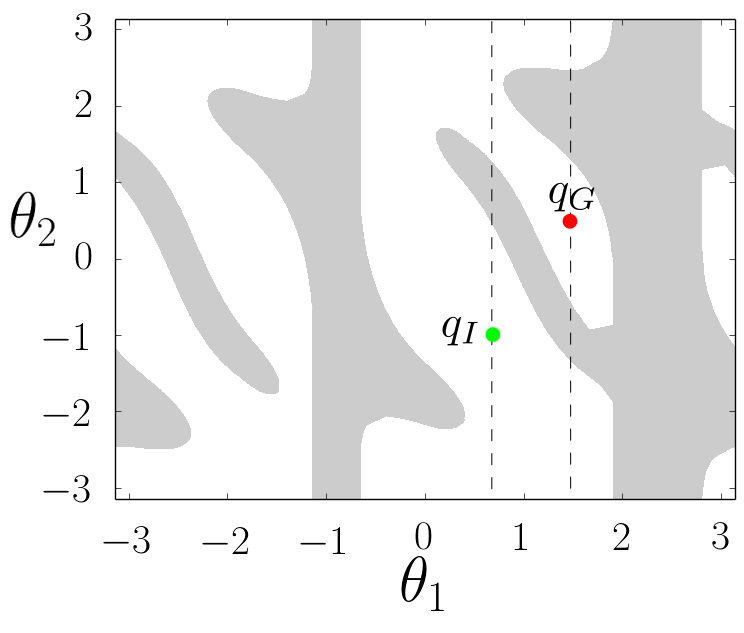}

\includegraphics[width=\lw\linewidth,height=\lh\linewidth]{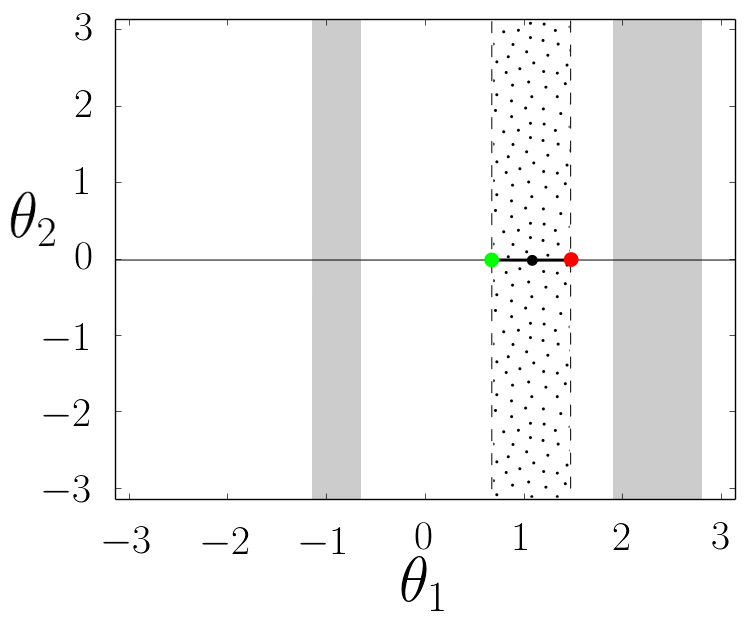}
\includegraphics[width=\lw\linewidth,height=\lh\linewidth]{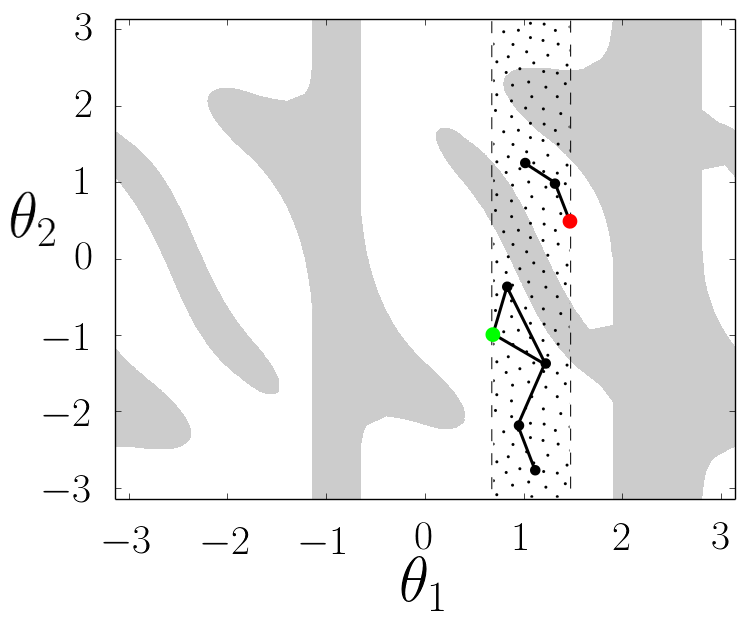}

\includegraphics[width=\lw\linewidth,height=\lh\linewidth]{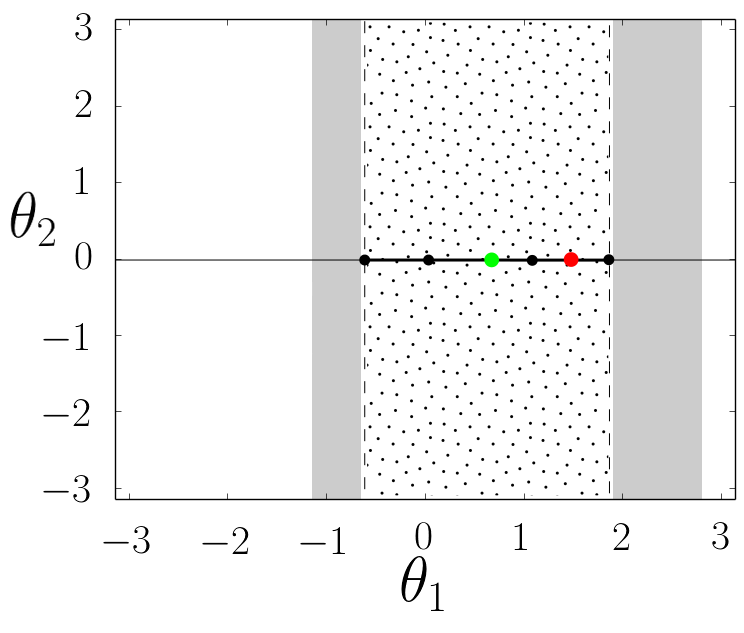}
\includegraphics[width=\lw\linewidth,height=\lh\linewidth]{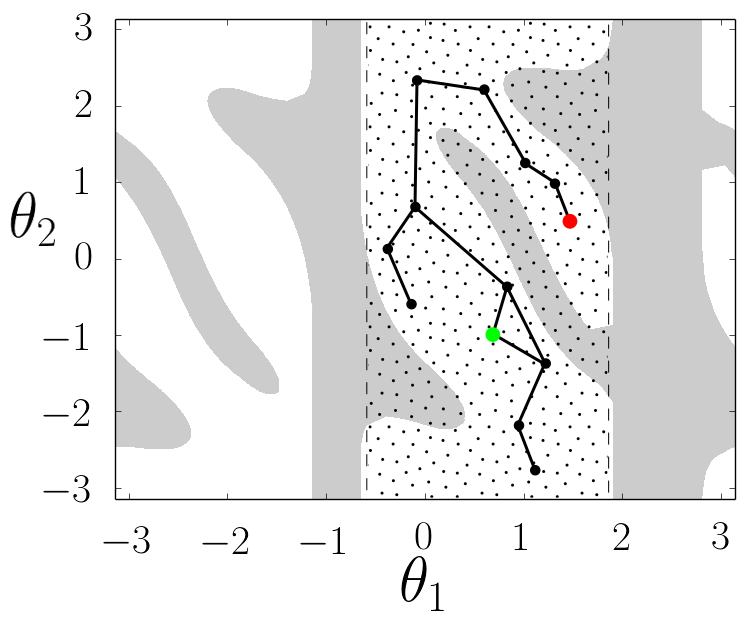}

\vspace{-10pt}
\caption{Schematic of \myplanner for a two-level decomposition. See text for clarification.\label{fig:qmp}}
\vspace{-10pt}
\end{figure}
\subsection{Quotient-Space Roadmap}

To simplify the algorithmic development, we group each quotient-space with its associated objects into a tuple called the quotient-space roadmap.


Let $\M$ be the $N$-dimensional configuration space of robot $\Robot$. We consider a nested
sequence of $K$ robots $\Robot_1 \subset \cdots \subset \Robot_K=\Robot$,
such that $\M$ is decomposed into a sequence of quotient spaces as $\M_1 \subset \cdots \subset \M_K = \M$. 

To each quotient space $\M_k$ we associate a start configuration $q^I_k$, a
goal configuration $q^G_k$, a graph $\G_k$, a shortest path $\path_k$ on $\G_k$
between $q^I_k$ and $q^G_k$, and a density $V_k$ defined on $\G_{k-1}\times \C_k$ as

\begin{equation}
  \begin{aligned}
    V_k &= \dfrac{|\G_k|}{\mu(\C_k)L_{k-1}}\label{eq:density}
  \end{aligned}
\end{equation}

whereby $|\G_k|$ are the number of vertices of $\G_k$, $L_{k-1}$ is the sum of all
edge lengths of $\G_{k-1}$, and $\mu(\C_k)$ is the $n^k$-dimensional measure of
$\C_k$. The density is used to decide which graph should be grown next.

We group all elements together into the quotient-space roadmap
\begin{equation}
  \begin{aligned}
    \Q_k &= \{ q^I_k, q^G_k, \C_k, \M_k, \G_k, \path_k, V_k, \Q_{k-1} \}\\
  \end{aligned}
\end{equation}

with $\C_k=\bigslant{\M_k}{\M_{k-1}}$, $\M_0=\emptyset$, $\Q_{0}=\emptyset$.

\subsection{Algorithmic Development}

\myplanner is an adapted version of the probabilistic roadmap planner (\prm) for quotient space decompositions. We first summarize the workings of \prm, then show how \myplanner can be built from it.

A simplified version of \prm is depicted in Algorithm \ref{alg:prm}. While a planner terminate condition (PTC) \footnote{A planner terminate condition (PTC) has to be chosen by a user and can be a time limit, an iterations limit, or a desired cost of the resulting path.} has not been reached (Line 2), the algorithm grows a graph $\G$ on the configuration space $\M$ of robot $\Robot$ (Line 3). If there exists a path on the graph between start and goal configuration (Line 4), then this path is returned (Line 5). If the PTC is reached then \prm fails and returns an empty path. 

The growing of the graph is depicted in Algorithm \ref{alg:prm_grow}. A configuration $\qr$ is sampled from the configuration space $\M$ (Line 1). If this configuration is valid (Line 2), then it is added to the graph (Line 3) and the nearest $R$ configurations $Q_R$ are searched (Line 4). The graph is extended in a straight line from each $\qn \in Q_R$ (Line 5) towards $\qr$ until it hits an obstacle (Line 6). The last configuration before the obstacle becomes $\qw$, and the edge between $\qn$ and $\qw$ is added to the graph (Line 7).

\myplanner is depicted in Algorithm \ref{alg:qmp}. 
An empty priority queue is constructed (Line 1), and the $k$-th quotient roadmap is initialized (Line 3) and added to the queue (line 4). While no path between start and goal has been found (Line 5), we pop the quotient roadmap with the smallest density from the queue (line 6), grow its graph (line 7) and push it back onto the queue (line 8). Then we check if there exists a path on the current quotient space (line 9); if yes, we construct its solution (line 10), and we continue to the next quotient roadmap. For $k=1$ the algorithm is equivalent to \prm. For $k>1$, multiple quotient spaces are inside the queue, and depending on the density function we pop one quotient space and grow its graph. The algorithm terminates if either the path $\path_K$ has been found, or if the PTC is reached, in which case $\path_K = \emptyset$ is returned. 

The growing of the quotient space graph is depicted in Algorithm \ref{alg:qmp_grow}. Instead of sampling $\M_k$ as in the \prm, we sample instead $\C_k$ uniformly, and we pick one configuration from the graph $\G_{k-1}$ on $\M_{k-1}$ (Line 1). The \texttt{SampleGraph} samples a uniform vertex from the graph $\G_{k-1}$. Then a random incoming edge is chosen, and a configuration uniformly on the edge is sampled. This is called Random-Vertex-Edge (RVE) sampling \cite{leskovec_2006}. The cartesian product $\circ$ merges the two configuration to yield a configuration on $\M_k$. The rest of the algorithm (line 2-7) operates as the \prm algorithm, with two exceptions. First, the R-Nearest-Neighbors method measures distance not by euclidean distance on $\M_k$, but by the graph distance on $\G_{k-1}$ plus euclidean distance on $\C_k$ (Line 4). Second, the \texttt{Connect} method does not interpolate along a straight line, but interpolates along the edges of the graph $\G_{k-1}$, while interpolating on $\C^k$ using a straight line. For each vertex crossed on $\G_{k-1}$ we add another configuration. The \texttt{Connect} method then returns a piece-wise linear (PL) path on $\M_{k}$. For each edge along this PL-path we add one edge to the graph $\G_k$ (Line 7).

Interpolating along the graph instead of using a straight line should be seen as a change in the metric on $\M_k$. While a standard euclidean metric is agnostic about the graph on $\M_{k-1}$, our graph interpolation metric utilizes the knowledge about $\G_{k-1}$ to improve the metric computation. 

\lstset{language=C}
\alglanguage{pseudocode}
\begin{algorithm}
  \caption{PRM($q_I,q_G,\M$)\cite{rrt}}
\begin{algorithmic}[1]
  \State \Call{init}{$\G,q_I$}
  \While{$\neg\Call{ptc}{}$}
    \State $\Call{GrowPRM}{\G,\M}$
    \If{$\Call{IsConnected}{q_I,q_G,\G}$}
    \State \Return \Call{Path}{$q_I,q_G,\G$}
    \EndIf
  \EndWhile
 \State \Return $\emptyset$
\end{algorithmic}
\label{alg:prm}
\end{algorithm}

\begin{algorithm}
  \caption{GrowPRM($\G,\M$)}
  \begin{algorithmic}[1]
    \State $\qr \gets \Call{Sample}{\M}$
    \If{$\neg\Call{IsValid}{\qr}$}
    	\Return 
	\EndIf        
    \State $\Call{add\_vertex}{\qr,\G}$
    \State $Q_R \gets \Call{R-NearestNeighbors}{\qr,\G}$
    \For{$\qn \in Q_R$}
      \State $\qw \gets \Call{Connect}{\qn,\qr}$
      \State $\Call{add\_edge}{\qn,\qw,\G}$
    \EndFor
  \end{algorithmic}
  \label{alg:prm_grow}
\end{algorithm}

\begin{algorithm}
  \caption{QMP($q^I_{1,\cdots,K},q^G_{1,\cdots,K},M_{1,\cdots,K}$)}
\begin{algorithmic}[1]
  \State $Q \gets \Call{priority\_queue}{}$
  \For{$k=1$ to $K$}
    \State $\Q_k=\Call{Init}{q^I_k, q^G_k, M_k,\Q_{k-1}}$
    \def\Qbest{\Q_{\text{least}}}
    \State $Q.\Call{push}{\Q_k}$
    \While{$\path^k == \emptyset$ and $\neg\Call{ptc}{}$}
      \State $\Qbest = Q.\Call{pop}{}$
      \State $\Call{GrowQMP}{\Qbest}$
      \State $Q.\Call{push}{\Qbest}$
      \If{$\Call{IsConnected}{\Q_k}$}
      	\State $\path_k = \Call{Path}{q^I_k,q^G_k,\Q_k}$
      \EndIf
    \EndWhile
  \EndFor
\State \Return $\path_K$
\end{algorithmic}
  \label{alg:qmp}
\end{algorithm}
\begin{algorithm}
\caption{Init($q^I_k,q^G_k,\M_k,\Q_{k-1}$)}\label{alg:qmpinit}
\begin{algorithmic}[1]
  \State \Return $\{q^I_k, q^G_k, \bigslant{\M_k}{\M_{k-1}}, \M_k, \emptyset, \emptyset, 0, \Q_{k-1} \}$
\end{algorithmic}
\end{algorithm}

\begin{algorithm}
  \caption{GrowQMP($\Q_k$)}\label{alg:qmpgrow}
  \begin{algorithmic}[1]
    \def\qr{q_{rand}}
    \def\qn{q_{near}}
    \def\qw{q_{new}}
    \State $\qr \gets \Call{SampleGraph}{\G_{k-1}} \circ \Call{Sample}{\C_k}$
    \If{$\neg\Call{IsValid}{\qr}$}
    	\Return 
	\EndIf        
    \State $\Call{add\_vertex}{\qr,\G}$
    \State $Q_R \gets \Call{R-NearestNeighbors}{\qr,\G_k,\G_{k-1}}$
    \For{$\qn \in Q_R$}
      \State $\{q_{n_1},\cdots,q_{n_J}\} \gets \Call{Connect}{\qn,\qr,\G_{k-1}}$
      \State $\Call{add\_edges}{q_{n_1},\cdots,q_{n_J},\G_k}$
    \EndFor
  \end{algorithmic}
  \label{alg:qmp_grow}
\end{algorithm}

\subsection{Implementation Details}

Our software uses the \klampt \cite{hauser_2016} physical simulator, and the open motion planning library \ompl \cite{sucan_2012}. 

 The nesting of robots has to be prespecified as a set of Unified Robot Description Format (URDF) files along with its subspaces. Each subspace is represented by an \ompl space, and our algorithm iterates through them computing the quotient spaces. We currently support the following quotient space computations: $\R^3 = \bigslant{\SE(3)}{\SO(3)}$, 
 $\R^2 = \bigslant{\SE(2)}{\SO(2)}$, 
 $\R^{n-m} = \bigslant{\R^n}{\R^m}$, 
 $\SE(3)= \bigslant{\SE(3)\times \R^n}{\R^n}$, and
 $\SE(3)\times \R^{n-m}= \bigslant{\SE(3)\times \R^n}{\SE(3)\times \R^m}$, with $n,m \in \N$ and $n>m>0$.

Our algorithm terminates after a path has been found on the configuration space, or a timelimit $T$ has been reached. 

\section{Probabilistic Completeness\label{sec:completeness}}

\def\Mt{\tilde{\M}}

A motion planning algorithm is probabilistically complete if the probability that the algorithm will find a path if one exists approaches one as the number of sampled points increases. We will show that \myplanner is probabilistically complete by alluring to the probabilistic completeness of \prm \cite{svestka_1996}.

The main difference of \myplanner and the \prm on the configuration space is the choice of a sampling sequence. For \prm the sampling sequence is dense in the configuration space, while it is not true for the sampling sequence generated by $\myplanner$. However, we will show that the sampling sequence by $\myplanner$ is dense in the space of feasible configurations, thus making $\myplanner$ probabilistically complete.

Let $\M$ be the configuration space and $U \subseteq \M$ be a subset. We say that $U$ is dense in $\M$ if $cl(U)=\M$ whereby $cl$ is the closure of $U$. Let $\alpha: \N \rightarrow \M$ be a random sequence. We say that the sequence $\alpha$ is dense in $\M$ if the set $\{\alpha_n,n\in \N\}$ is dense in $\M$ \cite{lavalle_2006}. 

Each quotient space $\M_1,\cdots,\M_K$ uses a sampling sequence as
$\alpha^{(k)}: \N \rightarrow \M_k$. We will show that this sequence is dense in $\Mt_k = \{ q\in \M_k| \V(q) \cap \env = \emptyset \}$, the feasible space of $\M_k$. 
\begin{theorem}
$\alpha^K$ is dense in $\Mt_K$
\end{theorem}

\begin{proof}

By induction for $K=1$ $\alpha^{(1)}$ is dense in $\M_1$, and since $\Mt_1 \subseteq \M_1$, $\alpha^{(1)}$ is dense in $\Mt_1$. For $K>1$, we assume $\alpha^{(K-1)}$ is dense in $\Mt_{K-1}$. Consider the sampling domain of $\alpha^{(K)}$, defined as $S = \{\alpha_n^{(K-1)} \times \C_K| \alpha_n^{(K-1)} \in \Mt_{K-1}, n \in \N\}$. Due to the necessary condition of nested robots (Theorem \ref{thm:necessary}), we have that if $r \circ p \in \Mt_K$ then $r \in \Mt_{K-1}$. Thus $\Mt_K \subseteq S$. Since $\alpha^{(K)}$ is dense in $S$ by definition, $\alpha^{(K)}$ is dense in $\Mt_K$.
\end{proof}

Finally, we need to make sure that every $\alpha^{(k)}$ is an infinite sequence. This is only guaranteed if each quotient space from the priority queue in Algorithm \ref{alg:qmp} is chosen infinitely many times. A sufficient condition is to choose a density function strictly monotonically increasing, as we did in Eq. \ref{eq:density}.


\section{Quotient-Space Heuristics\label{sec:heuristics}}

We first discuss intricacies of quotient-spaces, and then use this knowledge to design heuristics for \myplanner.

\subsection{Quotient-Space Intricacies}

Solution paths on quotient-spaces do not behave as nicely as one would expect. We show two examples where the solution path is non-simple or spurious. 

\subsubsection{Non-simple Paths on Quotient-Space}

A non-simple path is a path with self-intersections. We say that a path $\path: [0,1] \rightarrow \M$ is simple, if it is injective,
i.e. for any $t,s \in [0,1]$: if $\path(t)=\path(s)$ then $t=s$
\cite{docarmo_1976}. A path not being simple is called non-simple.  

While non-simple paths do not occur in configuration spaces, they can occur in quotient spaces. Fig. \ref{fig:nonsimple} shows an example, where a rectangular shaped planar robot needs to move from a start configuration (green) to a goal configuration (red). The
configuration space is $\SE(2)$ and the quotient space is $\R^2 = \bigslant{\SE(2)}{\SO(2)}$, obtained by nesting a disk inside the rectangle. The
solution path is simple on $\SE(2)$, but non-simple on $\R^2$. 

\begin{figure}
  \centering
  \includegraphics[width=0.32\linewidth,height=0.28\linewidth]{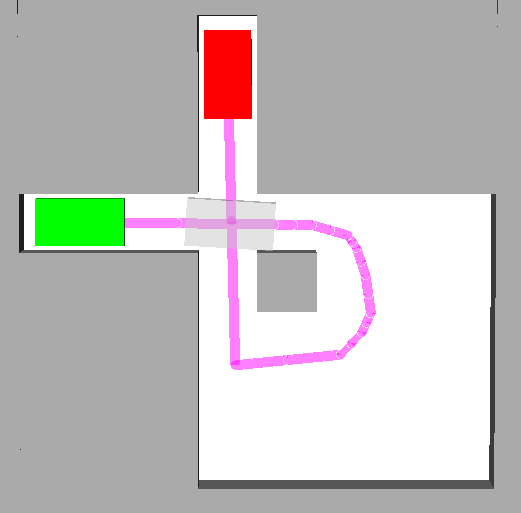}
  \includegraphics[width=0.32\linewidth,height=0.28\linewidth]{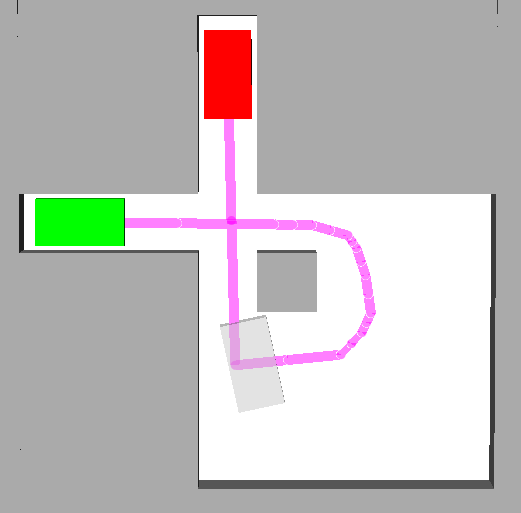}
  \includegraphics[width=0.32\linewidth,height=0.28\linewidth]{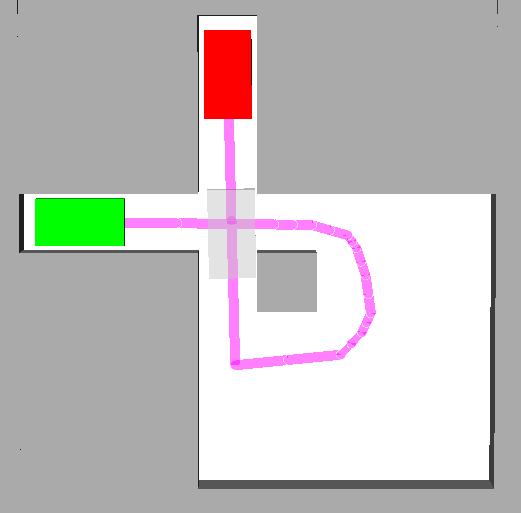}
\vspace{-8pt}
  \caption{Non-simple path on Quotient-Space: See text for clarification.\label{fig:nonsimple}}
\end{figure}

\subsubsection{Spurious Shortest Path on Quotient-Space}

Spurious paths are solution paths on the quotient-space which are infeasible on the configuration space. Let $\M_1$ be a quotient space of the configuration space $\M_2$ and let $\path^1: [0,1] \rightarrow \M_1$ be the shortest feasible path on the quotient space between $q_I^1$ and $q_G^1$. We say that $\path^1$ is feasible on the configuration space, if there exists a feasible path $\path^2: [0,1] \rightarrow \path^1 \times \C_2$. If $\path^1$ is infeasible on the configuration space, we call it spurious.

Consider the example in Fig. \ref{fig:spurious}, where an L-shaped robot needs to move through an environment with two passages above and below a rectangular obstacle (Left). The quotient-space is represented by a disk nested in the L-shaped robot. The shortest path on the quotient space is spurious (Middle), since the only feasible path goes above the obstacle (Right). 

\begin{figure}
  \centering
  \includegraphics[width=0.32\linewidth,height=0.28\linewidth]{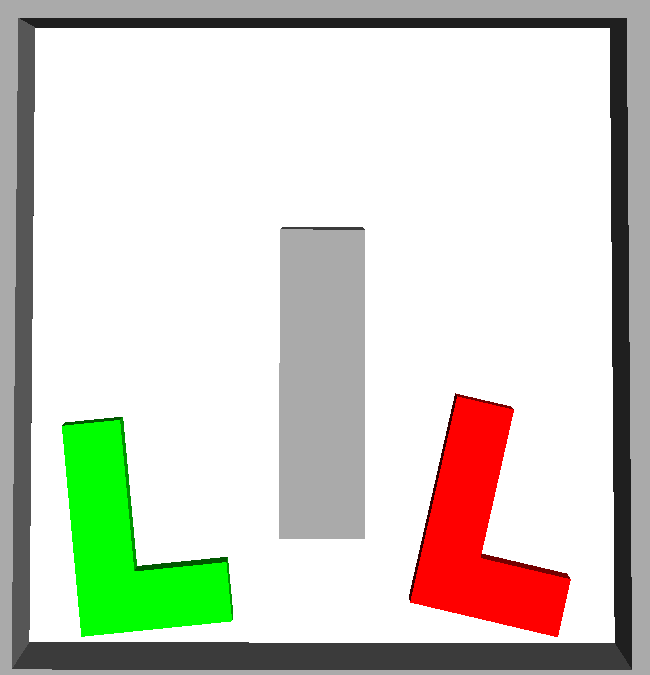}
  \includegraphics[width=0.32\linewidth,height=0.28\linewidth]{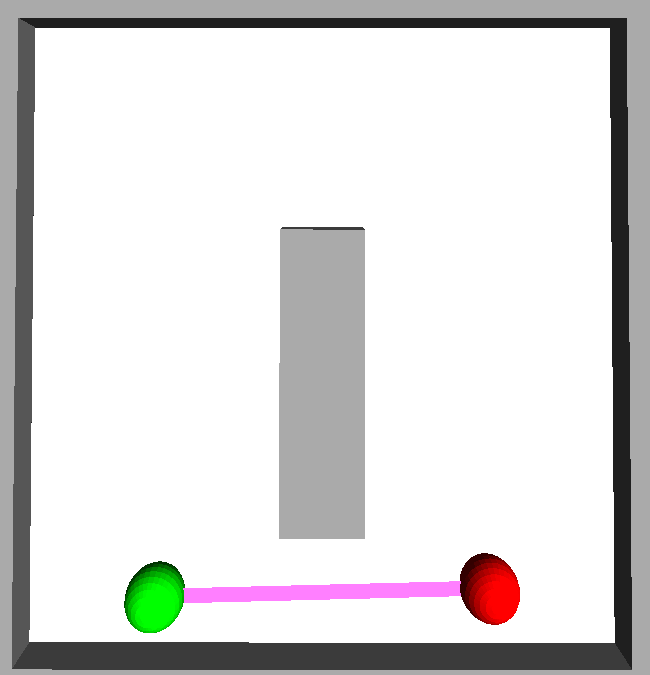}
  \includegraphics[width=0.32\linewidth,height=0.28\linewidth]{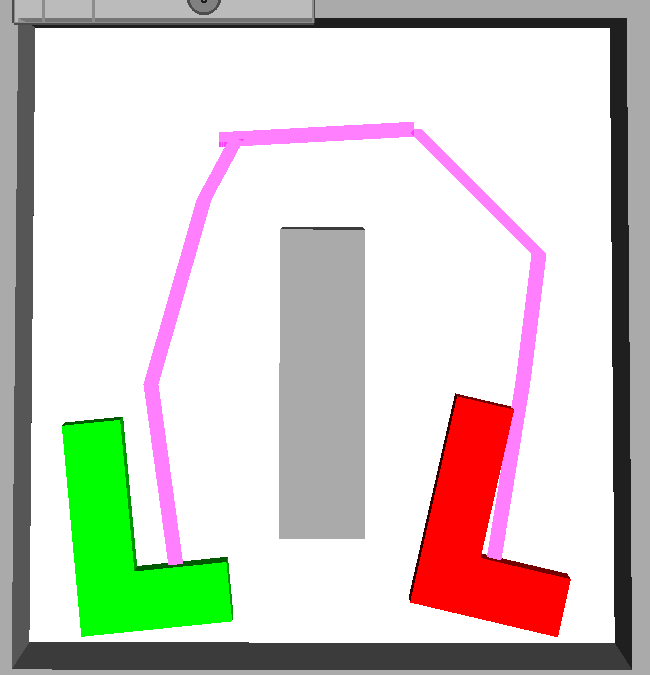}
\vspace{-8pt}
  \caption{Spurious shortest path in Quotient-Space: An environment for an L-shaped robot with configuration space $\SE(2)$ and quotient space $\R^2$. See text for clarification.\label{fig:spurious}}
\end{figure}

\subsection{Heuristics}

Keeping the intricacies in mind, we design three heuristics which we found to be beneficial to reduce the runtime of \myplanner.

\subsubsection{Diminishing-Time Shortest Path Sampling}

If the shortest path is spurious this often can be established quickly. We add a diminishing time heuristic to sample first the shortest path on the quotient-space with probability $p=0.8$ and decrease this probability over time with the amount of samples taken from the shortest path.

\subsubsection{Graph Thickening}

Often a solution path on the quotient-space is spurious, but a solution path nearby is not spurious and contains a solution on the configuration space. To alleviate this problem, we introduce an $\epsilon$-graph thickening. Given a random sample $q$ from the graph, we add an offset uniformly distributed from a ball around $q$ with radius $\epsilon$. This helps in finding nearby solutions. We have set $\epsilon = 0.01$ for free-floating robots and $\epsilon=0.1$ for fixed-base robots.

\subsubsection{Increasing Clearance}

In some cases it is advantageous to inflate the shape of the nested robot such that the clearance of the solution path on the quotient space is increased. One has to be careful with this heuristic, since it trades-off completeness with efficiency. We apply this only for the free-floating robots in our experiments by increasing the size of the inscribed sphere by $\delta=1.2$.

\section{Experiments\label{sec:experiments}}

We perform experiments on four different environments as shown in Fig. \ref{fig:experiments}: A free-floating rigid body, a free-floating articulated body, a fixed-base serial chain and a fixed-base tree chain robot. The problem is to find a path for each robot from the initial configuration (green) to the goal configuration (red). All experiments have been perfomed on a quad core 2.6 Ghz laptop with 31 GB working memory. We perform collision detection using the proximity query package (PQP) \cite{pqp}.


\subsection{Free-floating rigid body}

The double L-shape is a rigid body consisting of two L-shapes glued together. The environment is a wall with a small square hole in it, as depicted in Fig. \ref{fig:experiments}. This problem was introduced by \cite{berg_2005} as an example of a narrow passage planning problem. We decompose the double L-shape as $\R^3 \subset \SE(3)$, as depicted in Fig. \ref{fig:dls_decomposition}.  We compared our algorithm $\myplanner$ with three state of the art algorithms implemented in the \ompl software framework: \prm \cite{kavraki_1996}, bidirectional rapidly-exploring random tree (\rrtconnect) \cite{kuffner_2000} and the expansive space trees planner (\est) \cite{hsu_1997}. Fig. \ref{benchmark:dls} shows that only \myplanner and \prm solved all runs, and that \myplanner has a median runtime of $2.5$s, compared to $27$s for \prm.

\begin{figure}
\includegraphics[width=0.32\linewidth]{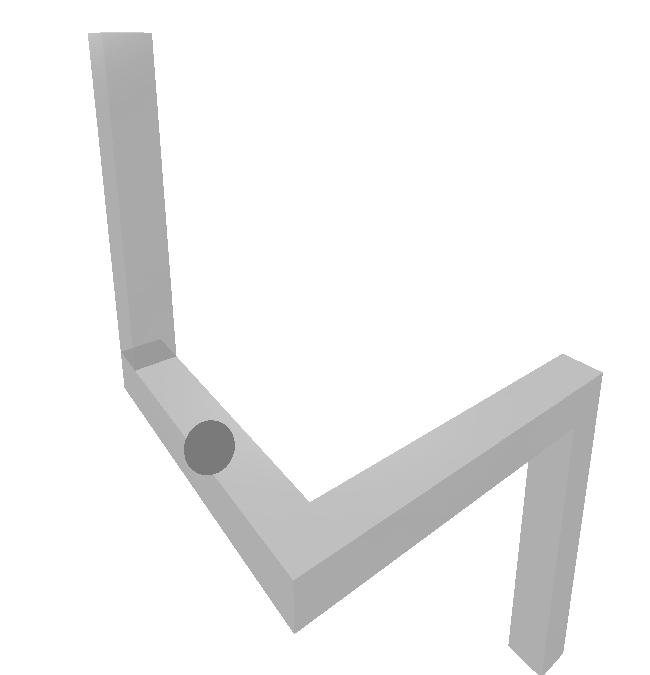}
\includegraphics[width=0.32\linewidth]{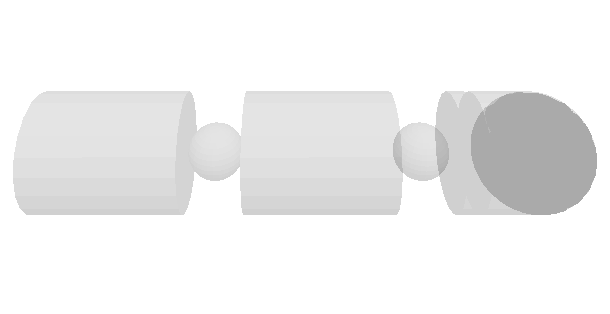}
\includegraphics[width=0.32\linewidth]{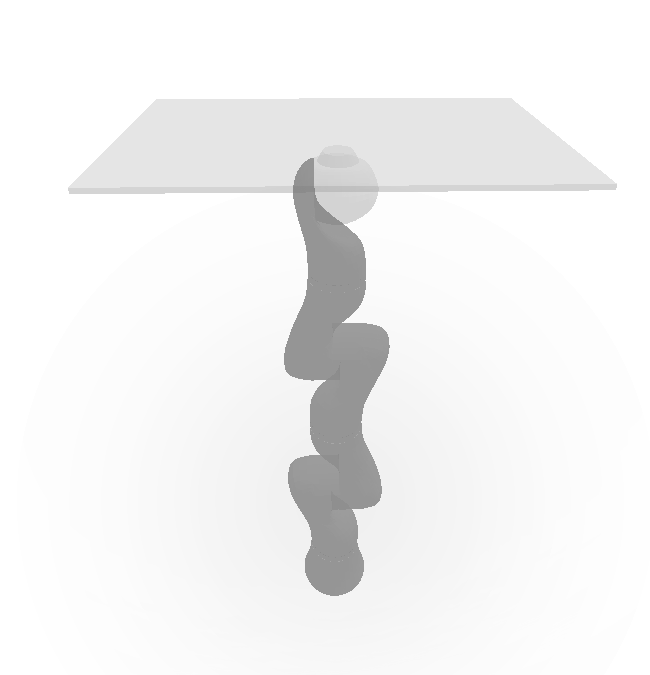}
\caption{Nesting of robots for the Double L-shape, Snake and KUKA LWR with windshield. The darker shade is the nested robot.\label{fig:dls_decomposition}}
\end{figure}

\subsection{Free-floating articulated body}

The mechanical snake is a 10-dof articulated body which has three links, connected each by two revolute joints with limits (See Fig. \ref{fig:dls_decomposition}). The snake can freely translate and rotate in space. We found the most effective decomposition to be $\R^3 \subset \SE(3)\times \R^4$. \myplanner has a median runtime of $30$s ($67$s for \est). However, we can see that \myplanner has two outliers. 

\subsection{Fixed-base serial chain: KUKA LWR}

The KUKA Ligthweight Robot (LWR) is a 7-dof manipulator. We consider transporting a windshield through a factory simulating a car manufacturing task. Our decomposition is $\R^5 \subset \R^7$. In the experiments, only \myplanner was able to solve all runs, with a median time of $18$s ($164$s for \est). 

\subsection{Fixed-base tree chain: Baxter}

Baxter is a two-arm fixed-base robot with a tree kinematic chain having two serial chains of 7 dof for each arm, and having a total of 14 dof. We consider a maintenance task, where Baxter needs to move its arms into small openings of a defect water tank. We decompose Baxter as 
$\R^5 \subset \R^7 \subset \R^{12} \subset \R^{14}$ as depicted in Fig. \ref{fig:baxter_decomposition}. Only \myplanner was able to solve all runs, with a median runtime of $4.5$s compared to $94$s for \rrtconnect.

\begin{figure}
\includegraphics[width=0.24\linewidth]{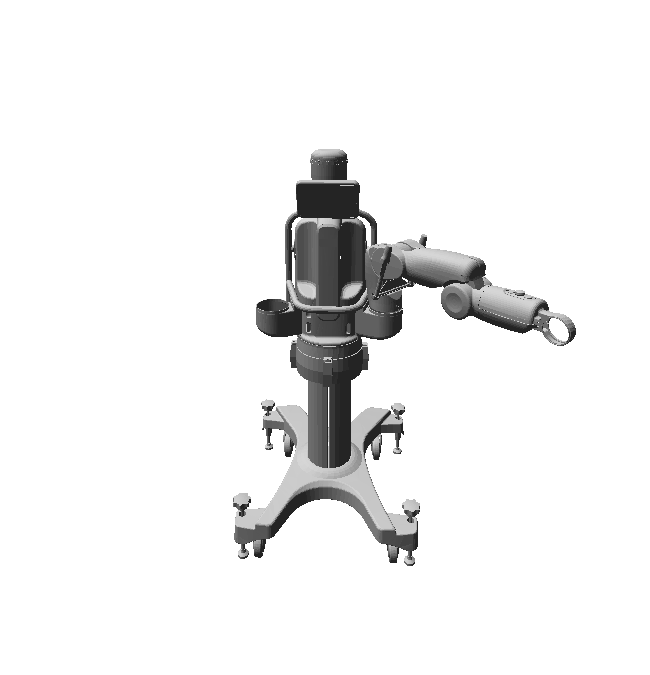}
\includegraphics[width=0.24\linewidth]{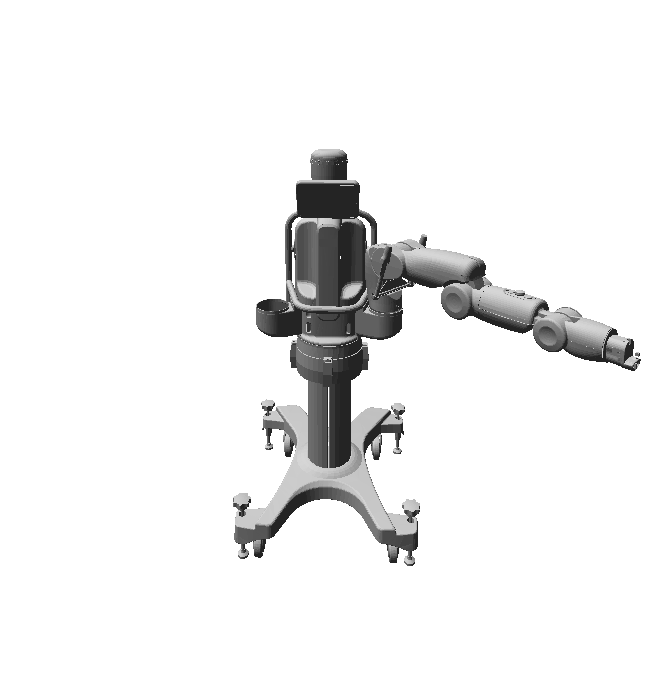}
\includegraphics[width=0.24\linewidth]{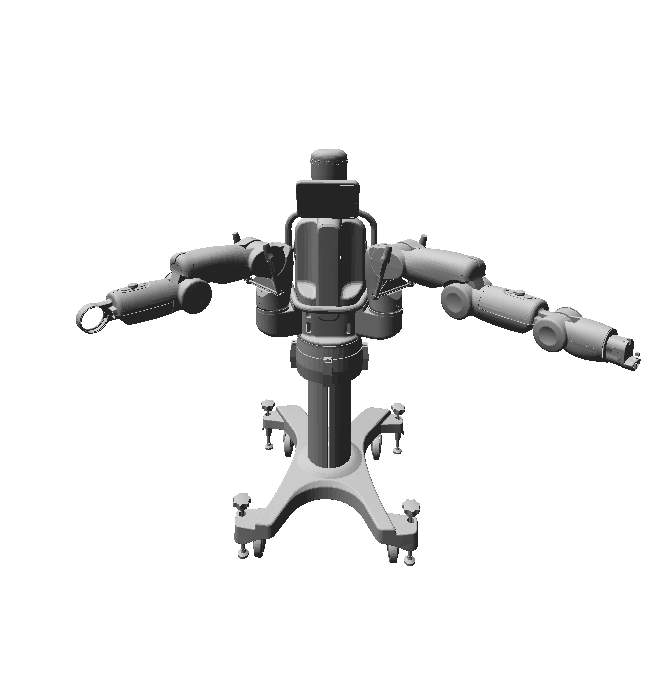}
\includegraphics[width=0.24\linewidth]{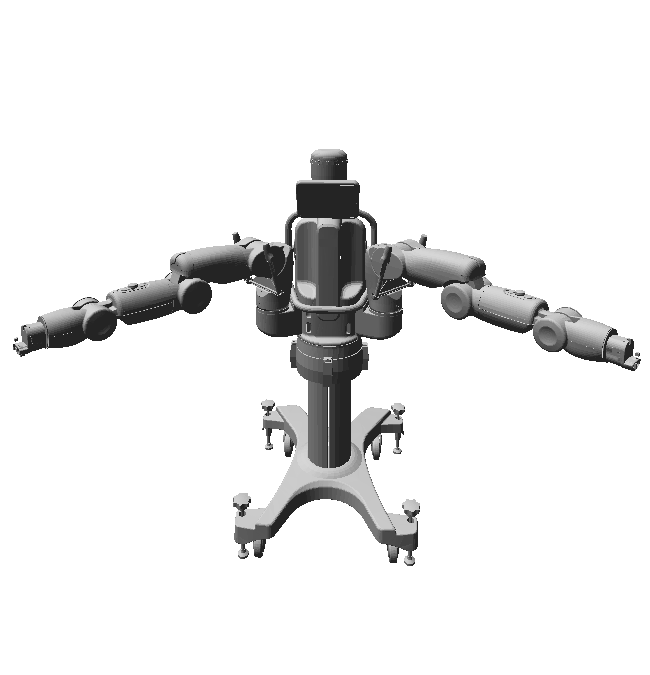}
\caption{Nesting of robots for Baxter Robot by Rethinking Robotics.\label{fig:baxter_decomposition}}
\end{figure}

\begin{figure}
  \centering
  \def\lw{0.48\linewidth}
  \def\lh{0.48\linewidth}
  \includegraphics[width=\lw,height=\lh]{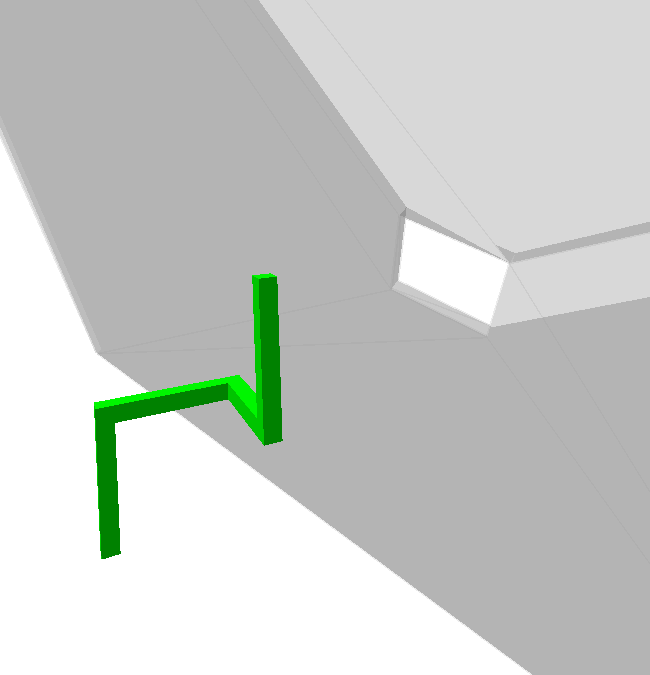}
  \includegraphics[width=\lw,height=\lh]{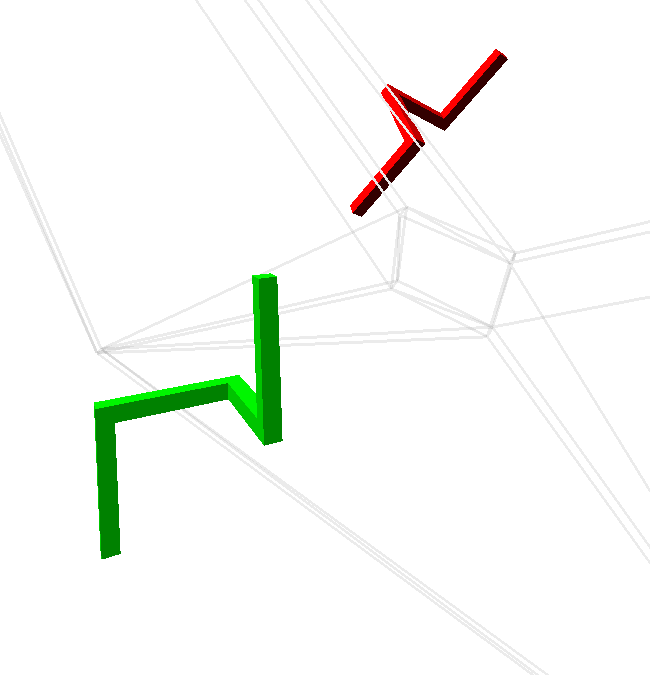}
  \includegraphics[width=\lw,height=\lh]{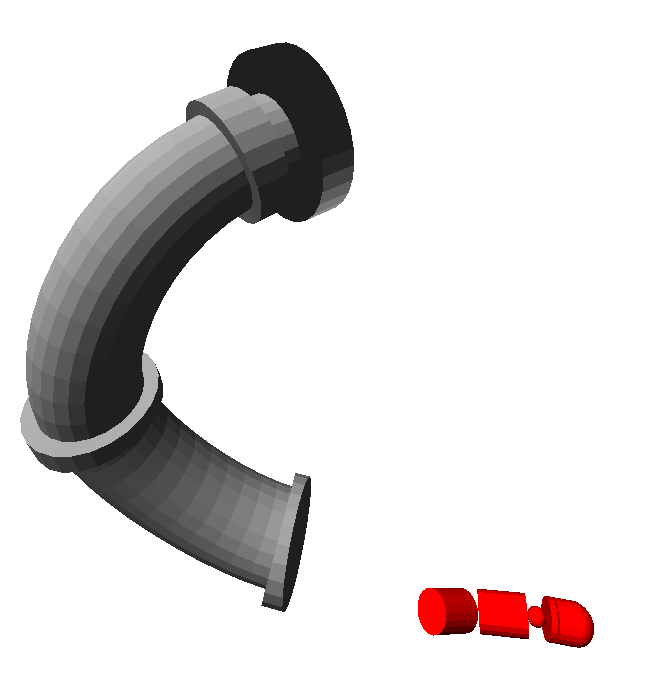}
  \includegraphics[width=\lw,height=\lh]{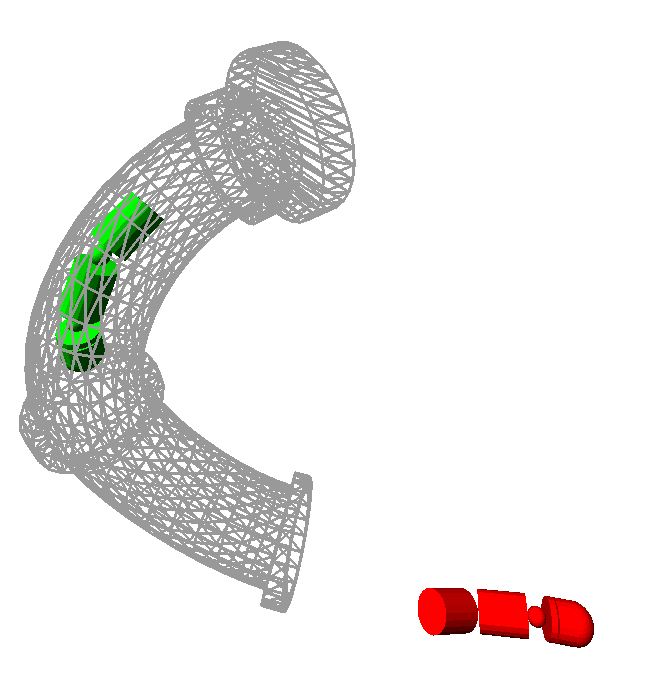}
  \includegraphics[width=\lw,height=\lh]{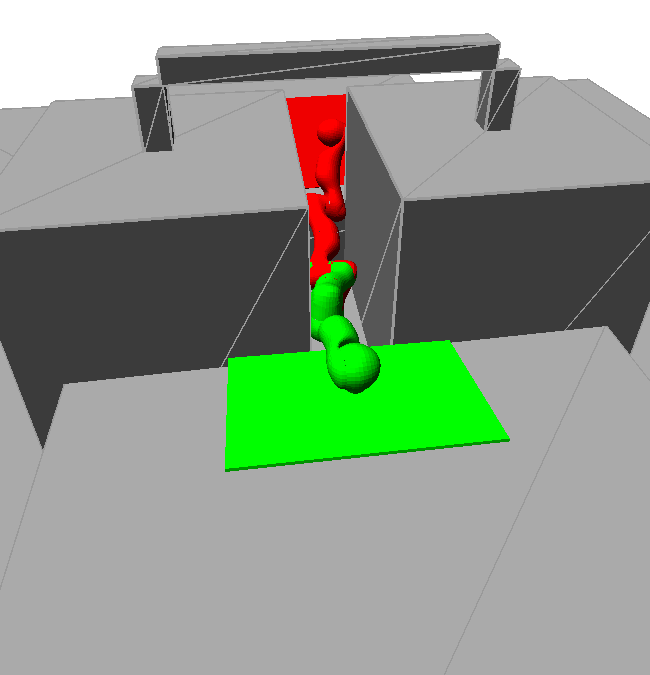}
  \includegraphics[width=\lw,height=\lh]{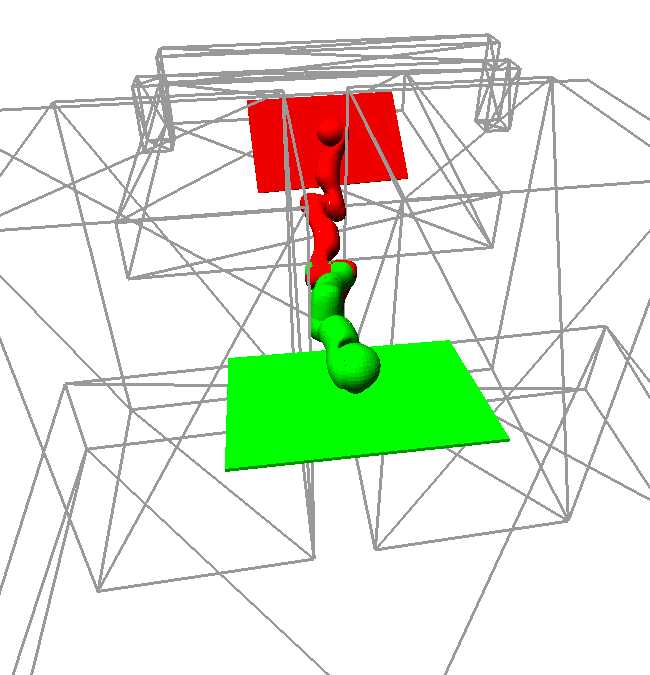}
  \includegraphics[width=\lw,height=\lh]{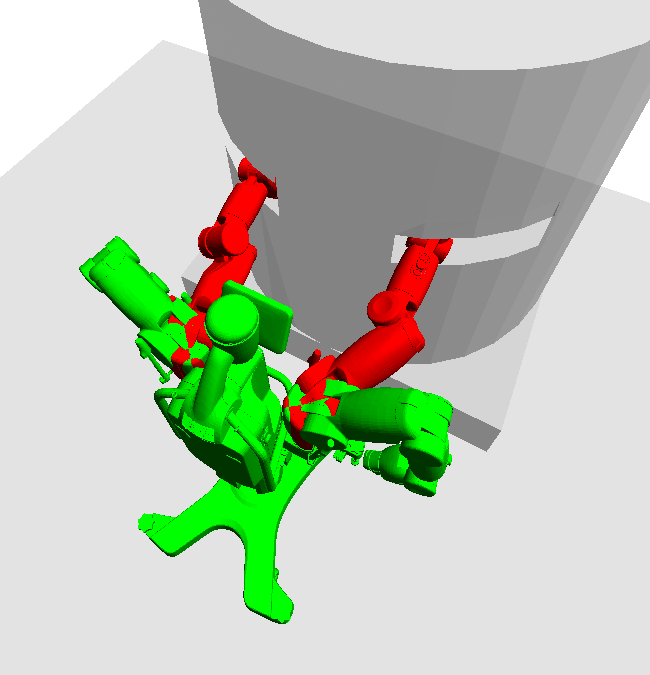}
  \includegraphics[width=\lw,height=\lh]{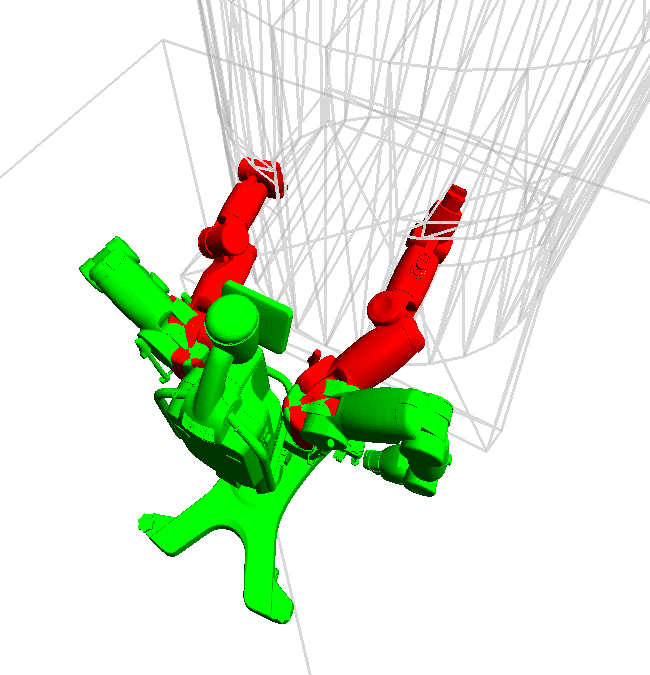}
\vspace{-8pt}
  \caption{The four experiments considered in this paper: For each experiment we show the start configuration of the robot (green) and the goal configuration (red). The left column shows the faces of the environment mesh, while the right column shows the edges of the environment mesh. \label{fig:experiments}}
\end{figure}

\begin{figure}
  \centering
  \includegraphics[width=0.95\linewidth,height=0.9\linewidth]{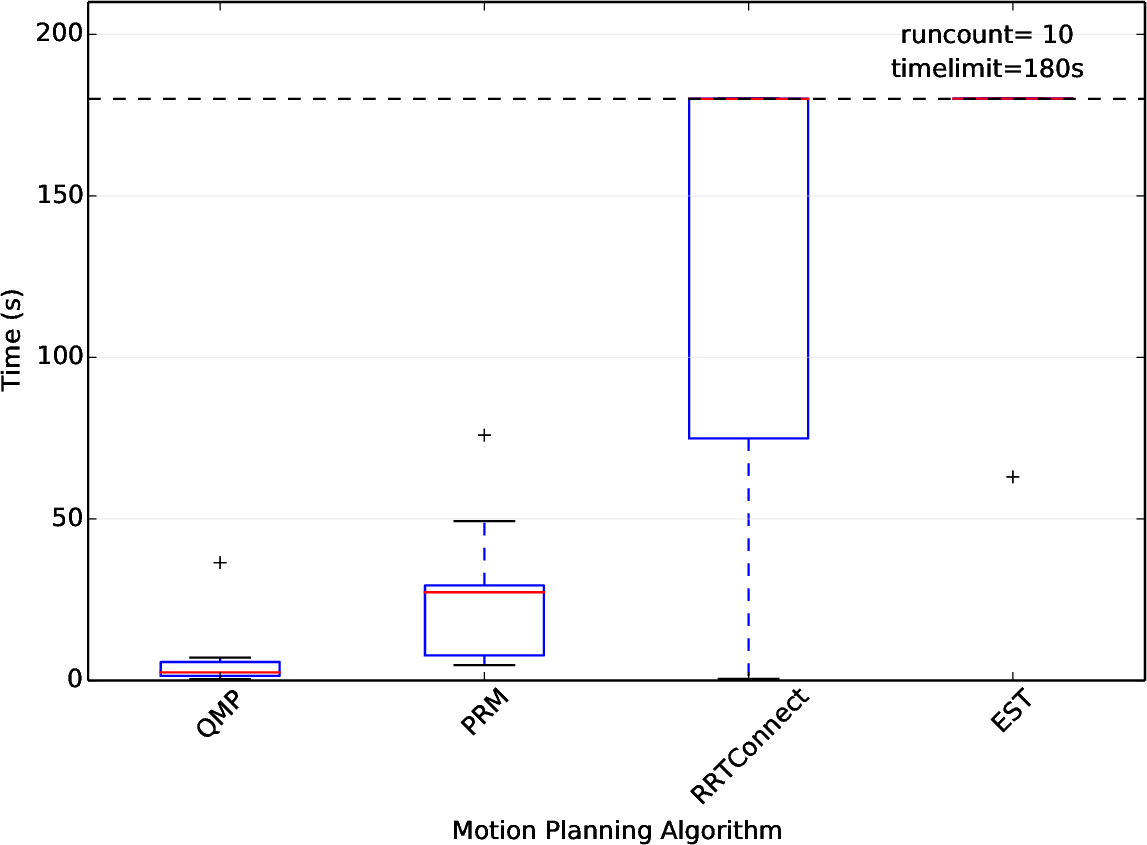}
\vspace{-8pt}
  \caption{Free-floating rigid body: Double L-Shape 6-dof benchmark.\label{benchmark:dls}}
\end{figure}

\begin{figure}
  \centering
  \includegraphics[width=0.95\linewidth,height=0.9\linewidth]{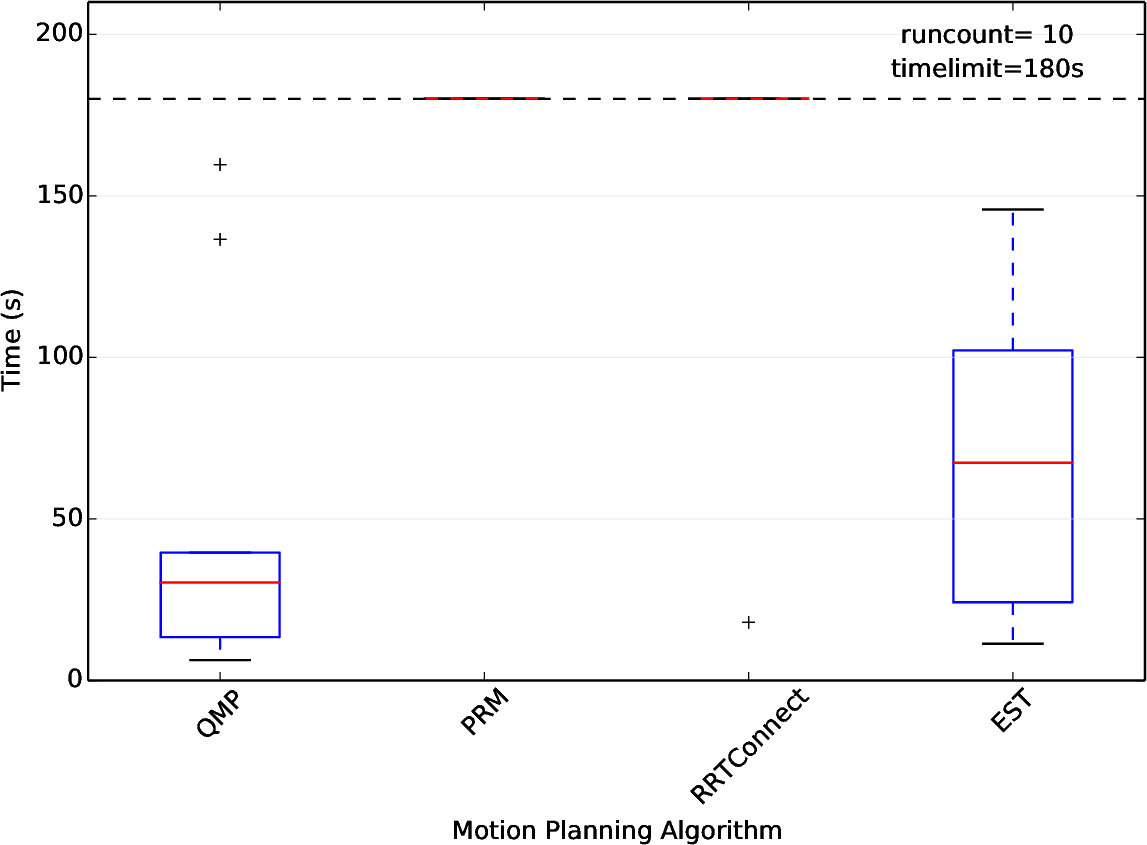}
\vspace{-8pt}
  \caption{Free-floating articulated body: Mechanical Snake 10-dof benchmark.\label{benchmark:snake}}
\end{figure}

\begin{figure}
  \centering
  \includegraphics[width=0.95\linewidth,height=0.9\linewidth]{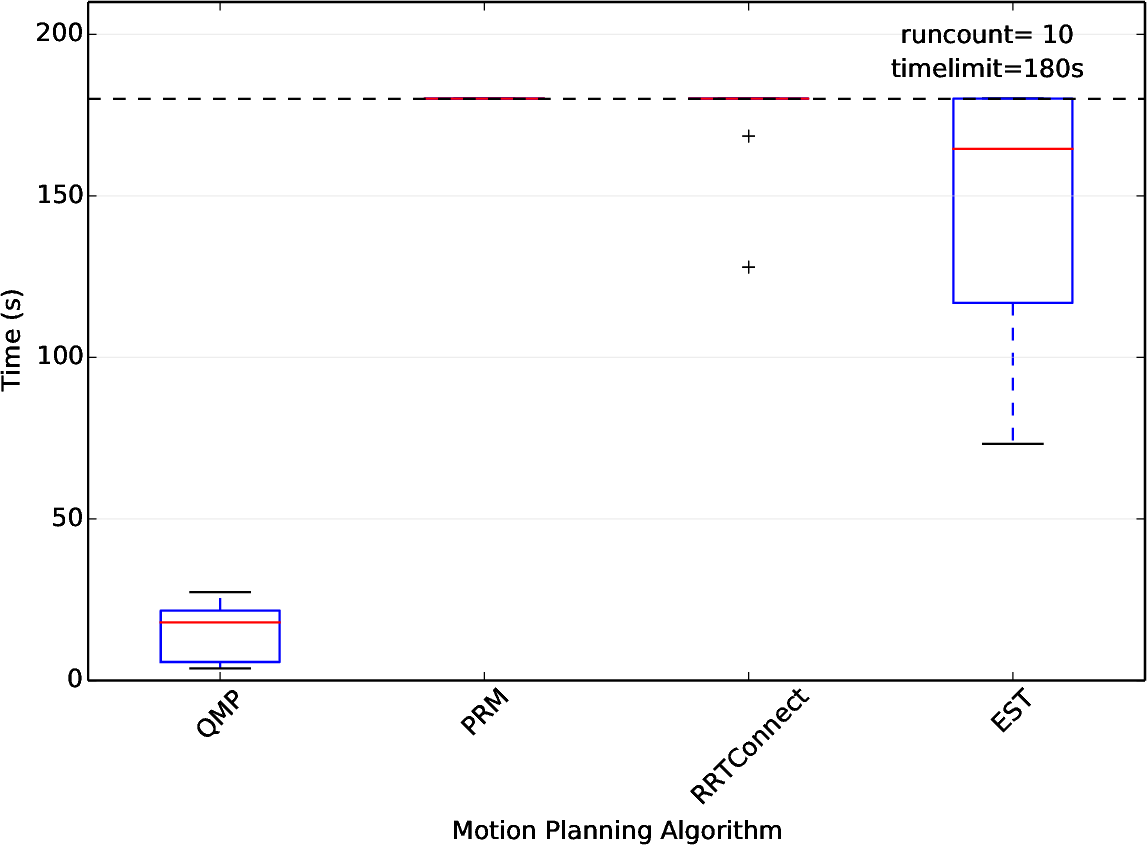}
\vspace{-8pt}
  \caption{Fixed-base serial chain: KUKA LWR 7-dof benchmark.\label{benchmark:kuka}}
\end{figure}

\begin{figure}
  \centering
  \includegraphics[width=0.95\linewidth,height=0.9\linewidth]{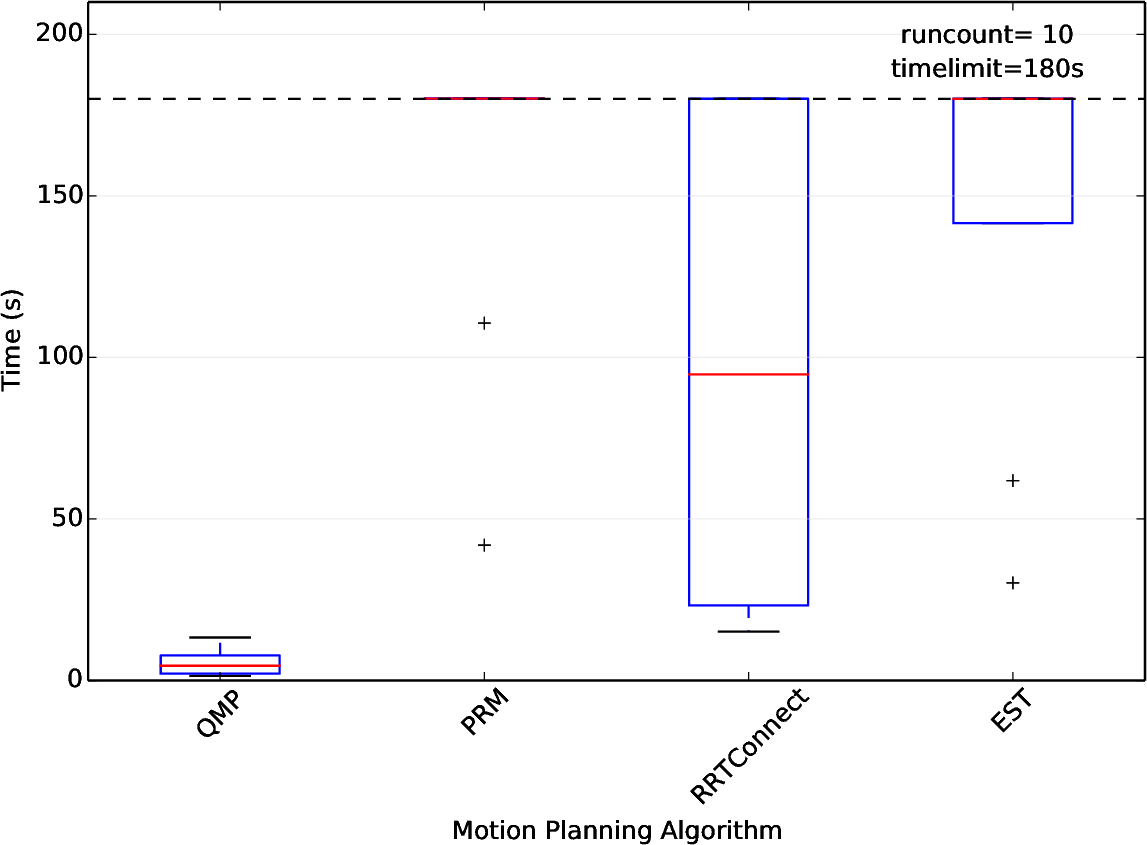}
\vspace{-8pt}
  \caption{Fixed-base tree chain: Baxter 14-dof benchmark.\label{benchmark:baxter}}
\end{figure}
\section{Conclusion}

We have introduced the quotient-space decomposition, a configuration space decomposition based on nested robots. To exploit this decomposition we developed the Quotient-space roadMap Planner (\myplanner). We showed that \myplanner is probabilistically complete and that \myplanner has lower runtime on four environments compared to state-of-the-art motion planning algorithms implemented in the \ompl software.

We like to extend \myplanner in three directions.
First, the quotient-space decomposition of a configuration space is not unique, and has to be specified by a human operator. We like to investigate which decomposition is optimal, and automate its specification. Second, we like to generalize our approach to closed kinematic chains, constraint manifolds, and dynamic constraints.
Finally, we like to apply the quotient-space decomposition to environments changing over time, where a roadmap on the quotient space could improve fast decision making.

\bibliographystyle{unsrtnat}
{\small
\bibliography{bib/general,bib/hierarchy}

\begin{thebibliography}{27}
\providecommand{\natexlab}[1]{#1}
\providecommand{\url}[1]{\texttt{#1}}
\expandafter\ifx\csname urlstyle\endcsname\relax
  \providecommand{\doi}[1]{doi: #1}\else
  \providecommand{\doi}{doi: \begingroup \urlstyle{rm}\Url}\fi

\bibitem[LaValle(2006)]{lavalle_2006}
S.~M. LaValle.
\newblock \emph{\href{http://planning.cs.uiuc.edu/}{Planning Algorithms}}.
\newblock Cambridge University Press, 2006.

\bibitem[Reif(1979)]{reif_1979}
John~H Reif.
\newblock
  \href{http://ieeexplore.ieee.org/xpls/abs_all.jsp?arnumber=4568037}{Complexity
  of the mover's problem and generalizations}.
\newblock In \emph{Conference on Foundations of Computer Science}, pages
  421--427, 1979.

\bibitem[Munkres(2000)]{munkres_2000}
James Munkres.
\newblock
  \emph{\href{https://books.google.co.jp/books/about/Topology.html?id=XjoZAQAAIAAJ}{Topology}}.
\newblock Pearson, 2000.

\bibitem[Zhang and Zhang(2004)]{zhang_2004}
Ling Zhang and Bo~Zhang.
\newblock \href{https://link.springer.com/chapter/10.1007/3-540-39205-X_2}{The
  quotient space theory of problem solving}.
\newblock \emph{Fundamenta Informaticae}, 59\penalty0 (2-3):\penalty0 287--298,
  2004.

\bibitem[Yao et~al.(2013)Yao, Vasilakos, and Pedrycz]{yao_2013}
Jing~Tao Yao, Athanasios~V Vasilakos, and Witold Pedrycz.
\newblock \href{https://ieeexplore.ieee.org/document/6479257/}{Granular
  computing: perspectives and challenges}.
\newblock \emph{Transactions on Cybernetics}, 43\penalty0 (6):\penalty0
  1977--1989, 2013.

\bibitem[Bretl(2006)]{bretl_2006}
Timothy Bretl.
\newblock
  \href{http://journals.sagepub.com/doi/10.1177/0278364906063979}{Motion
  planning of multi-limbed robots subject to equilibrium constraints: The
  free-climbing robot problem}.
\newblock \emph{International Journal of Robotics Research}, 25\penalty0
  (4):\penalty0 317--342, 2006.

\bibitem[Grey et~al.(2017)Grey, Ames, and Liu]{grey_2017}
Michael~X Grey, Aaron~D Ames, and C~Karen Liu.
\newblock \href{https://ieeexplore.ieee.org/document/7989551/}{Footstep and
  motion planning in semi-unstructured environments using randomized
  possibility graphs}.
\newblock In \emph{IEEE International Conference on Robotics and Automation},
  2017.

\bibitem[Tonneau et~al.(2018)Tonneau, Prete, Pettré, Park, Manocha, and
  Mansard]{tonneau_2018}
S.~Tonneau, A.~D. Prete, J.~Pettré, C.~Park, D.~Manocha, and N.~Mansard.
\newblock \href{https://ieeexplore.ieee.org/document/8341955/}{An Efficient
  Acyclic Contact Planner for Multiped Robots}.
\newblock \emph{Transactions on Robotics}, 34\penalty0 (3):\penalty0 586--601,
  June 2018.

\bibitem[Zhang et~al.(2009)Zhang, Pan, and Manocha]{zhang_2009}
Liangjun Zhang, Jia Pan, and Dinesh Manocha.
\newblock \href{https://ieeexplore.ieee.org/document/5379545/}{Motion planning
  of human-like robots using constrained coordination}.
\newblock In \emph{IEEE International Conference on Humanoid Robots}, 2009.

\bibitem[Rickert et~al.(2014)Rickert, Sieverling, and Brock]{rickert_2014}
Markus Rickert, Arne Sieverling, and Oliver Brock.
\newblock \href{https://ieeexplore.ieee.org/document/6871370/}{Balancing
  exploration and exploitation in sampling-based motion planning}.
\newblock \emph{Transactions on Robotics}, 30\penalty0 (6):\penalty0
  1305--1317, 2014.

\bibitem[Kunz et~al.(2016)Kunz, Thomaz, and Christensen]{kunz_2016}
Tobias Kunz, Andrea Thomaz, and Henrik Christensen.
\newblock \href{https://ieeexplore.ieee.org/document/7487120/}{Hierarchical
  rejection sampling for informed kinodynamic planning in high-dimensional
  spaces}.
\newblock In \emph{IEEE International Conference on Robotics and Automation},
  pages 89--96. IEEE, 2016.

\bibitem[Gammell et~al.(2014)Gammell, Srinivasa, and Barfoot]{gammell_2014}
Jonathan~D Gammell, Siddhartha~S Srinivasa, and Timothy~D Barfoot.
\newblock \href{https://ieeexplore.ieee.org/document/6942976/}{Informed RRT*:
  Optimal sampling-based path planning focused via direct sampling of an
  admissible ellipsoidal heuristic}.
\newblock In \emph{IEEE International Conference on Intelligent Robots and
  Systems}, pages 2997--3004. IEEE, 2014.

\bibitem[Van~den Berg and Overmars(2005)]{berg_2005}
Jur~P Van~den Berg and Mark~H Overmars.
\newblock \href{https://ieeexplore.ieee.org/document/1307191/}{Using workspace
  information as a guide to non-uniform sampling in probabilistic roadmap
  planners}.
\newblock \emph{International Journal of Robotics Research}, 24\penalty0
  (12):\penalty0 1055--1071, 2005.

\bibitem[Zucker et~al.(2008)Zucker, Kuffner, and Bagnell]{zucker_2008}
Matt Zucker, James Kuffner, and J~Andrew Bagnell.
\newblock \href{https://ieeexplore.ieee.org/document/4543787/}{Adaptive
  workspace biasing for sampling-based planners}.
\newblock In \emph{IEEE International Conference on Robotics and Automation},
  pages 3757--3762. IEEE, 2008.

\bibitem[Orthey et~al.(2015)Orthey, Lamiraux, and Stasse]{orthey_2015}
Andreas Orthey, Florent Lamiraux, and Olivier Stasse.
\newblock \href{https://ieeexplore.ieee.org/document/7139695/}{Motion Planning
  and Irreducible Trajectories }.
\newblock In \emph{IEEE International Conference on Robotics and Automation},
  2015.

\bibitem[{\c{S}}ucan and Kavraki(2009)]{sucan_2009}
Ioan~A {\c{S}}ucan and Lydia~E Kavraki.
\newblock
  \href{http://link.springer.com/chapter/10.1007/978-3-642-00312-7_28}{Kinodynamic
  motion planning by interior-exterior cell exploration}.
\newblock In \emph{Algorithmic Foundation of Robotics VIII}. Springer, 2009.

\bibitem[Lee(2003)]{lee_2003}
John~M. Lee.
\newblock
  \emph{\href{https://www.springer.com/jp/book/9780387217529}{Introduction to
  Smooth Manifolds}}.
\newblock Springer New York, New York, NY, 2003.

\bibitem[Leskovec and Faloutsos(2006)]{leskovec_2006}
Jure Leskovec and Christos Faloutsos.
\newblock \href{https://ieeexplore.ieee.org/document/7023975/}{Sampling from
  large graphs}.
\newblock In \emph{Proceedings of the 12th ACM SIGKDD international conference
  on Knowledge discovery and data mining}, pages 631--636. ACM, 2006.

\bibitem[Lavalle and Kuffner~Jr(2000)]{rrt}
Steven~M Lavalle and James~J Kuffner~Jr.
\newblock
  \href{http://citeseerx.ist.psu.edu/viewdoc/summary?doi=10.1.1.36.7457}{Rapidly-Exploring
  Random Trees: Progress and Prospects}.
\newblock In \emph{Algorithmic and Computational Robotics: New Directions},
  2000.

\bibitem[Hauser(2016)]{hauser_2016}
Kris Hauser.
\newblock
  \href{https://link.springer.com/chapter/10.1007/978-3-319-28872-7_21}{Robust
  contact generation for robot simulation with unstructured meshes}.
\newblock In \emph{International Journal of Robotics Research}, pages 357--373.
  Springer, 2016.

\bibitem[{\c{S}}ucan et~al.(2012){\c{S}}ucan, Moll, and Kavraki]{sucan_2012}
Ioan~A {\c{S}}ucan, Mark Moll, and Lydia Kavraki.
\newblock \href{https://ieeexplore.ieee.org/document/6377468/}{The open motion
  planning library}.
\newblock \emph{Robotics and Automation Magazine}, 2012.

\bibitem[Svestka(1996)]{svestka_1996}
Petr Svestka.
\newblock
  \emph{\href{http://citeseerx.ist.psu.edu/viewdoc/summary?doi=10.1.1.19.3881}{On
  probabilistic completeness and expected complexity for probabilistic path
  planning}}, volume 1996.
\newblock Utrecht University: Information and Computing Sciences, 1996.

\bibitem[Do~Carmo(1976)]{docarmo_1976}
Manfredo~Perdigao Do~Carmo.
\newblock
  \emph{\href{https://books.google.co.jp/books/about/Differential_Geometry_of_Curves_and_Surf.html?id=v4vqjgEACAAJ}{Differential
  geometry of curves and surfaces}}, volume~2.
\newblock Prentice-hall Englewood Cliffs, 1976.

\bibitem[Larsen et~al.(1999)Larsen, Gottschalk, Lin, and Manocha]{pqp}
Eric Larsen, Stefan Gottschalk, Ming~C Lin, and Dinesh Manocha.
\newblock
  \href{http://citeseerx.ist.psu.edu/viewdoc/summary?doi=10.1.1.39.8240}{Fast
  proximity queries with swept sphere volumes}.
\newblock Technical report, Technical Report TR99-018, Department of Computer
  Science, University of North Carolina, 1999.

\bibitem[Kavraki et~al.(1996)Kavraki, Svestka, Latombe, and
  Overmars]{kavraki_1996}
Lydia~E Kavraki, Petr Svestka, J-C Latombe, and Mark~H Overmars.
\newblock \href{https://ieeexplore.ieee.org/document/508439/}{Probabilistic
  roadmaps for path planning in high-dimensional configuration spaces}.
\newblock \emph{Transactions on Robotics}, 12\penalty0 (4):\penalty0 566--580,
  1996.

\bibitem[Kuffner and LaValle(2000)]{kuffner_2000}
James~J Kuffner and Steven~M LaValle.
\newblock \href{https://ieeexplore.ieee.org/document/844730/}{RRT-connect: An
  efficient approach to single-query path planning}.
\newblock In \emph{IEEE International Conference on Robotics and Automation},
  volume~2, pages 995--1001. IEEE, 2000.

\bibitem[Hsu et~al.(1997)Hsu, Latombe, and Motwani]{hsu_1997}
David Hsu, J-C Latombe, and Rajeev Motwani.
\newblock \href{https://ieeexplore.ieee.org/document/619371/}{Path planning in
  expansive configuration spaces}.
\newblock In \emph{IEEE International Conference on Robotics and Automation},
  volume~3, pages 2719--2726. IEEE, 1997.

\end{thebibliography}
}

\end{document}